\documentclass[twoside]{article}
\usepackage[left=2cm,right=2cm,top=2cm,bottom=2cm]{geometry}
\usepackage{multicol}
\usepackage{amsmath}
\usepackage{amsfonts}
\usepackage{amssymb}
\usepackage{graphicx}
\usepackage{color}
\usepackage[utf8]{inputenc}
\usepackage{amsfonts}
\usepackage{color}
\usepackage[normalem]{ulem}
\usepackage{nomencl}
\makenomenclature

\newcommand{\yadi}{\nomenclature}
\newenvironment{proof}{\noindent{\sc Proof.}}{\qed}
\newtheorem{theorem}{Theorem}[section]
\newtheorem{lemma}{Lemma}[section]
\newtheorem{cor}{Corollary}[section]
\newtheorem{remark}{Remark}[section]
\newtheorem{definition}{Definition}[section]
\newtheorem{prop}{Proposition}[section]
\newtheorem{uda}{Example}[section]
\newcommand{\qed}{$\blacksquare$}

\def\bhag#1{\noindent
\setcounter{equation}{0}
\section{#1}
}
\newcommand\gattha[2]{\genfrac{}{}{0pt}{}{#1}{#2}}

\def\HH{{\mathbb H}}

\def\RR{{\mathbb R}}
\def\CC{{\mathbb C}}
\def\ZZ{{\mathbb Z}}
\def\SS{{\mathbb S}}

\def\TT{\mathbb T}

\def\bs#1{{\boldsymbol{#1}}}
\def\x{\mathbf{x}}
\def\k{\mathbf{k}}
\def\y{\mathbf{y}}
\def\u{\mathbf{u}}
\def\w{\mathbf{w}}

\def\v{\mathbf{v}}

\def\O{{\cal O}}

\def\C{{\mathcal C}}

\def\YY{\mathbb{Y}}

\def\esssup{\mathop{\hbox{\textrm{ess sup}}}}

\def\be{\begin{equation}}
\def\ee{\end{equation}}
\def\bea{\begin{eqnarray}}
\def\eea{\end{eqnarray}}

\def\disp{\displaystyle}

\def\donchitre#1#2{\vskip 6.5cm\noindent
\parbox[t]{1in}{\special{eps:#1.eps x=6.5cm y=5.5cm}}
\hbox to 7cm{}\parbox[t]{0.0cm}{\special{eps:#2.eps x=6.5cm y=5.5cm}}}

\def\XX{{\mathbb X}}
\def\BB{{\mathbb B}}

\def\bs#1{{\boldsymbol{#1}}}

\def\gs{\gtrsim}
\def\ls{\lesssim}

\title{Approximation by non-symmetric networks for cross-domain learning}
\author{
 H.~N.~Mhaskar\thanks{
Institute of Mathematical Sciences, Claremont Graduate University, Claremont, CA 91711. 
\textsf{email:} hrushikesh.mhaskar@cgu.edu.
The research is  supported in part by NSF grant DMS 2012355, and ONR grants N00014-23-1-2394, N00014-23-1-2790..}
 }
 \date{}
\begin{document}
\maketitle
\begin{abstract}
For the past 30 years or so, machine learning has stimulated a great deal of research in the study of approximation capabilities (expressive power) of a multitude of processes, such as approximation by shallow or deep neural networks, radial basis function networks, and a variety of kernel based methods. 
Motivated by applications such as invariant learning, transfer learning, and synthetic aperture radar imaging, we initiate in this paper a general approach to study the approximation capabilities of kernel based networks using non-symmetric kernels. 
While singular value decomposition is a natural instinct to study such kernels, we consider a more general approach to include the use of a family of kernels, such as generalized translation networks (which include neural networks and translation invariant kernels as special cases) and rotated zonal function kernels.
Naturally, unlike traditional kernel based approximation, we cannot require the kernels to be positive definite.
In particular, we obtain estimates on the accuracy of uniform  approximation of  functions in a Sobolev class by  ReLU$^r$ networks when $r$ is not necessarily an integer. 
Our general results apply to the approximation of functions with small smoothness compared to the dimension of the input space.
\end{abstract}



\bhag{Introduction}\label{bhag:intro}
We will provide general introductory remarks in Section~\ref{bhag:genintro}, followed by a more technical discussion of a motivating example involving approximation by shallow, periodic, ReLU networks in Section~\ref{bhag:techintro}. 
The outline of the paper is given in Section~\ref{bhag:outline}.
\subsection{General introduction}\label{bhag:genintro}
A fundamental problem of machine learning is the following. 
Given data of the form $\{(x_j,y_j)\}$, where $y_j$'s are noisy samples of an unknown function $f$ at the points $x_j$'s, find an approximation to $f$. 
Various tools are used for the purpose; e.g., deep and shallow neural networks, radial basis function (RBF) and other kernel based methods.
Naturally, there is a great deal of research about the dependence of the accuracy in approximation on the mechanism used for approximation (e.g., properties of the activation funtion for the neural networks, or the kernel in kernel based methods), the dimension of the input data, the complexity of the model used (e.g., the number of nonlinearities in a neural or RBF network), smoothness of the function, and other factors, such as the size of the coefficients and weights of a neural network.
In \cite{dingxuanpap}, we have argued that the study of approximation capabilities of a deep network can be reduced to that of the capabilities of shallow networks; the advantage of using deep networks stemming from the fact that they can exploit any inherent compositional structure in the target function, which a shallow network cannot. 
There are also other efforts \cite{belkin2018understand} to argue that an understanding of kernel based networks is essential for an understanding of deep learning.
In this paper, we therefore focus on the approximation capabilities of shallow networks.
It is not difficult to extend these results to deep networks using the ideas in \cite{dingxuanpap}.

A (shallow) neural network with activation function $G$ is a function on a Euclidean space $\RR^q$ of the form $\x\mapsto \sum_{j=1}^M a_jG(\x\cdot\w_j+b_j)$, where $\w_j\in\RR^q$, $a_j, b_j\in\RR$.
We note that by ``dimension lifting'', i.e., by writing $\mathbf{X}=(\x,1)$, $\mathbf{W}_j=(\w_j,b_j)$, we can express a neural network in the form $\sum_{j=1}^M a_jG(\mathbf{X}\cdot\mathbf{W}_j)$.
More generally, 
 a kernel based network with kernel $G$ has the form $\x\mapsto \sum_{j=1}^M a_jG(\x,\w_j)$, where $\x, \w_j\in\RR^q$ and $a_j\in\RR$. 

A natural class of functions to be approximated by kernel based networks is the class of functions of the form
\be\label{eq:functionform}
f(x)=\int_\XX G(x,y)d\tau(y),
\ee
where $\XX$ is the input space, and $\tau$ is a signed measure on $\XX$ having bounded total variation.
The integral expression in \eqref{eq:functionform} and the class of functions are sometimes called an infinite (or continuous) network and the variational space respectively. 
A lucid account of functions satisfying \eqref{eq:functionform} from the point of view of reproducing kernel Banach spaces is given in \cite{bartolucci2023understanding}.
In the context of RBF kernels, the class of functions is known as the native space for the kernel $G$.
In this paper, we will use the term ``native space of $G$'' more generally to refer to the class of functions satisfying \eqref{eq:functionform}, and the term ``infinite network'' to denote the integral expression in that equation.

In the literature on approximation theory, it is customary to take $\XX$ to be the unit cube of $\RR^q$ or the torus $\TT^q=\RR^q/(2\pi\ZZ)^q$ or the unit (hyper)-sphere embedded in $\RR^{q+1}$. 
It is unrealistic to assume that in most practical machine learning problems, the data is actually spread all over these domains.
The so called manifold hypothesis assumes that the data is drawn from some probability distribution supported on some sub-manifold $\XX$ of dimension $q$ embedded in a high dimensional ambient space $\RR^Q$. 
There are some recent algorithms proposed to test this hypothesis \cite{fefferman2016testing}.
The manifold itself is not known, and a great deal of research in this theory is devoted to studying the geometry of this manifold.
For example, there are some recent efforts to approximate an atlas on the manifold using deep networks (e.g., \cite{coifman_deep_learn_2015bigeometric, chui_deep, relu_manifold_chen2019,schmidt2019deep}). 
A more classical approach is to approximate the eigenvalues and eigenfunctions of the Laplace-Beltrami operator on the manifold  using the so called graph Laplacian that can be constructed directly from the data (e.g., \cite{niyogi2, belkinfound, lafon, singer}). 
Starting with \cite{mauropap}, this author and his collaborators carried out an extensive investigation of function approximation on manifolds (e.g., \cite{eignet, modlpmz, compbio, heatkernframe, mhaskar2020kernel}).
During this research, we realized that the full strength of differentiability structure on the manifold is not necessary for studying function approximation.
Our current understanding of the properties of $\XX$ which are important for this purpose is encapsulated in the Definition~\ref{def:ddrdef} of data spaces.

There are two approaches to studying approximation bounds for functions using a kernel based network.
One approach is to treat the infinite network as an expectation of a family of random variables of the form $|\tau|_{TV}G(\x,\circ)h(\circ)$ with respect to the probability measure $|\tau|/|\tau|_{TV}$, where $h$ is the Radon-Nikodym derivative of $\tau $ with respect to $|\tau|$, and use concentration inequalities to obtain a discretized kernel based network. 
The approximation bounds in terms of the size $M$ of the network obtained in this way are typically independent of the dimension of the input space $\XX$, but are limited to the native spaces, e.g., \cite{klusowski2016uniform, barron1993, kurkova1, kurkova2, mhaskar2020dimension}.
We will refer to this approach as the probability theory approach.
Another approach which leads to dimension dependent bounds for more classical function spaces, such as the Sobolev classes, is the following.
We first approximate $f$ by a ``diffusion polynomial'' $P$ (cf. Section~\ref{bhag:ddr} for details). 
This polynomial is trivially in the native space. 
Using special properties of a Mercer expansion of $G$ and quadrature formulas, one approximates $P$ by kernel based networks (e.g., \cite{mhaskar1995degree, mnw2, zfquadpap, eignet, mhaskar2020kernel}). 
This approach usually requires $G$ to be a positive definite kernel, in the sense that all the coefficients in the Mercer expansion are positive.  Some tricks can be used to circumvent this restriction in special cases, such as ReLU networks (e.g., \cite{bach2014,sphrelu}).
We will refer to this approach as the approximation theory approach.
The probability theory approach relies entirely on very elementary properties of $G$, such as its supremum norm and Lipschitz continuity.
The approximation theory approach takes into account a more detailed structure of $G$ as well as the smoothness properties of the functions in a class potentially much larger than the native space. 

In all of the works which we are familiar with so far, the kernel $G$ is a symmetric kernel. 
There are many applications, where it is appropriate to consider non-symmetric kernels.
We give a few examples.
\begin{itemize}
\item In transfer learning, we wish to use the parameters trained on one data set  living on a space $\YY$ to learn a function on another data set living on the space $\XX$.
In this case, it is natural to consider a kernel $G:\XX\times\YY\to\RR$.
\item Another example is motivated by synthetic aperture radar imaging, where the observations have the same form as \eqref{eq:functionform}, but the integration is taken over a different set $\YY$ than the argument $x\in\XX$ \cite{cheney2009fundamentals}. 
The set $\YY$ represents the target from which the radar waves are reflected back and $\XX$ is the space defined by the beamformer/receiver.
\item In recent years, there is a growing interest in random Fourier features \cite{rahimi2007random}.
For example, a computation of the Gaussian kernel $\exp(-|\x_j-\x_k|^2/2)$ for every pair of points from $\{\x_j\}_{j=1}^M$ would take $\O(M^2)$ flops.
Instead, one evaluates a rectangular matrix of the form $A_{j,\ell}=\exp(i\x_j\cdot\omega_\ell)$ (the random features) for a large number of random samples $\omega_\ell$ drawn from the standard normal distribution. 
The kernel can be computed more efficiently using the tensor product structure of the features and a Monte-Carlo discretization of the expected value of the matrix $AA^T$.
When $\XX$ is a data space, the inner product has no natural interpretation, and one needs to consider more general random processes. 
Of course, the measure space for the probability measure generating the random variables is different from $\XX$.
Rather than computing the expected value of a product of two such processes, it is natural to wonder whether one could approximate a function directly using these modified random features. 
\item In image analysis, we may have to predict the label of an image based on data that consists of images rotated at different angles.
It is then natural to look for kernels of the form $(1/m)\sum_{\ell=1}^m  G(\x,R_\ell \y)$ where $R_\ell$'s are the rotations involved in the data set (e.g., \cite{poggio2016visual}).
\item In an early effort to study neural networks and translation invariant kernel based networks in a unified manner, we introduced the notion of a generalized translation network (GTN) in \cite{mhaskar1995degree}. 
Let $q\ge d\ge 1$ be integers.
The notion of generalized translation networks involves a family of kernels of the form $G(A_\ell\x-\y)$, where $\x\in\TT^q=\RR^q/(2\pi\ZZ)^q$, $\y\in \TT^d$, and $A_\ell$'s are $d\times q$ matrices with integer entries. 
The work \cite{mhaskar1995degree} gives some rudimentary bounds on approximation by GTN's, but the topic was not studied further in the literature as far as we are aware.
\end{itemize}

In this paper, we study approximation properties of  non-symmetric kernels in different settings: generalized translation networks (Example~\ref{uda:gentransnet}), zonal function networks including rotations as described above (Example~\ref{uda:sbf}), and  general non-symmetric kernels,  e.g., defining random processes on data spaces via Karhunen-Lo\'eve expansions (Example~\ref{uda:randomfeature}). 
We will prove a ``recipe theorem'' (Theorem~\ref{theo:maintheo}) which leads to these settings in a unified, but technical, manner. 
Together with approximation of functions in the native class, we will also study the question of simultaneous approximation of certain ``derivatives'' (cf. Definition~\ref{def:deroperator}) of the function by the corresponding derivatives of the networks themselves.

In recent years, ReLU networks and power ReLU networks, which use an activation function of the form $t\mapsto\max(0,t)^r$ have been studied widely.
From an approximation theory point of view our paper \cite{multilayer} is an early paper where a multivariate spline is expressed explicitly as a deep ReLU$^r$ network, so that approximation by such networks is immediately reduced to that by spline approximation.
In  \cite{bach2023relationship}, it is demonstrated how certain SBF and RBF kernels can be viewed in connection with approximation by shallow ReLU$^r$ networks. 
Some other relevant recent papers are, for example, \cite{xu2020finite, siegel2022sharp, klusowski2016uniform, ma2022uniform, mhaskar2020dimension, mhaskar2023tractability}.
We will illustrate our theory for these networks separately in Section~\ref{bhag:relu}.

Apart from dealing with non-symmetric kernels, some of the novelties of  our paper are the following.
\begin{itemize}
\item Although our analysis applies equally well to symmetric kernels, we do not require the kernels to be positive definite.
\item We combine the probability theory approach with the approximation theory approach focusing on approximation of ``rough'' functions on data spaces; i.e., functions for which the smoothness parameter is substantially smaller than the input dimension.
This means that we don't use the smoothness properties of the kernel $G$ strongly enough to use quadrature formulas, but use an initial approximation by diffusion polynomials as in the approximation theory approach, followed by ideas from probability theory.
\item The results are novel when applied to   zonal function networks with activation function $t_+^r$.
As far as we are aware, the known results are  for the native space, whose nature is not well understood since the kernels are not positive definite. 
Moreover, it is not clear whether (and which) Sobolev spaces can be characterized as intermediate spaces between the space of continuous functions and the native space for these activation functions.
\item Our results hold for such networks even when $r$ is not an integer.
\end{itemize}

\subsection{Technical introduction}\label{bhag:techintro}
In this section, we discuss an example to motivate the general theory described in the rest of the paper. 
For simplicity of exposition, the notation used in this section may not be the same as in the rest of the paper.

For integer $q\ge 1$ Let $\TT^q$ be the quotient space $\RR^q/(2\pi\ZZ)^q$.
The space of continuous functions on $\TT^q$ (i.e., the space of continuous functions on $\RR^q$ which are $2\pi$-periodic in each variable) is denoted by $C(\TT^q)$, and is equipped with the uniform norm $\|\cdot\|$.
If $\gamma>0$, then there is a unique integer $r$ and $\alpha\in (0,1]$ such that $\gamma=r+\alpha$. 
In this section only, the space $W_{q,\gamma}$ consists of $f\in C(\TT^q)$ which has $r$ derivatives with respect to each variable, and each of the derivatives $\mathcal{U}(f)$ of order $r$ satisfies
$$
|\mathcal{U}(f)(\x+\mathbf{h})+\mathcal{U}(f)(\x-\mathbf{h})-2 \mathcal{U}(f)(\x)|\le L|\mathbf{h}|_\infty^\alpha,
$$
where  the addition in $\x\pm\mathbf{h}$ is interpreted modulo $2\pi$ and $|\cdot|_p$ denotes the $\ell^p$ norm for a vector.
We note that the above condition can be expressed also in the form
\be\label{eq:modcont}
\omega_2(\mathcal{U}(f), \delta)=\sup_{|\mathbf{h}|_\infty\le \delta}\|\mathcal{U}(f)(\circ+\mathbf{h})+\mathcal{U}(f)(\circ-\mathbf{h})-2\mathcal{U}(f)\|_\infty \le L\delta^\alpha,
\ee
which can be generalized easily to other $L^p$ norms\footnote{The function $\omega_2$ is often referred to as the modulus of smoothness or second order modulus of continuity, the suffix $2$ referring to the fact that a second order difference is involved in the definition.}.
Let $\HH_n^q$ be the space of all trigonometric polynomials of spherical order $<n$; i.e.,
$$
\HH_n^q=\mathsf{span}\{\exp(i\k\cdot \circ) : |\k|_2<n\}.
$$
The central quantity of interest in approximation theory is
$$
E_{q,n}(f)=\min_{P\in\HH_n^q}\|f-P\|.
$$
It is well known (cf. \cite{timanbk}) that
$E_{q,n}(f)=\O(n^{-\gamma})$ \textbf{if and only if} $f\in W_{q,\gamma}$. 

We may therefore define a norm on $W_{q,\gamma}$ by
\be\label{eq:introsobnorm}
\|f\|_{W_{q,\gamma}} =\|f\| +\sup_{n\ge 1} n^\gamma E_{q,n}(f).
\ee
A popular activation function in the study of neural networks is the ReLU function: $t_+=\max(t,0)$. This is the solution of the initial value problem
$$
u''=\delta_0, \qquad u(0)=u'(0)=0,
$$
where $\delta_0$ is the Dirac delta supported at $0$.
The periodic analogue of this function is the Bernoulli spline given by
$$
\Gamma(t)=\sum_{\gattha{j\in\ZZ}{j\not=0}}\frac{e^{ijt}}{j^2}.
$$
Accordingly, the periodic ReLU network has the form $\sum_{\k: |\k|_\infty\le M} a_\k\Gamma(\k\cdot \circ -b_\k)$ for some real numbers $a_\k$, $b_\k$. 
In this section, we denote the set of all such networks by $\mathcal{G}_M$.

In this section, we examine the approximation of $f\in C(\TT^q)$ from $\mathcal{G}_M$. 
Analogously to the degree of approximation by trigonometric polynomials, we define (in this section only) \yadi{$\mathcal{E}_M(f)$}{Degree of approximation of $f$ by $M$-term neural networks, \eqref{eq:neuraldegree}}
\be\label{eq:neuraldegree}
\mathcal{E}_M(f)=\inf_{G\in \mathcal{G}_M}\|f-G\|.
\ee

Perhaps, the most popular way to study this problem is to consider the subspace $V$ of $C(\TT^q)$, known as the variation (or native) space for $\Gamma$. 
This is the space of all functions $f$ which can be expressed in the form
\be\label{eq:introvariation}
f(\x)=\int \Gamma(\k\cdot\x -b)d\mu(\k,b),
\ee
for some measure $\mu$ defined on $\ZZ^q\times\RR$, which has a finite total variation $|\mu|_{TV}$. 
In order to obtain bounds on uniform approximation, it is customary to assume that the measure $\mu$ is supported on some compact set, say $([-K,K]\cap\ZZ)^q\times [-a,a]$. 
An example of the pure probabilistic approach alluded to in the introduction is the following.

Using H\"offding's inequality and the Lipschitz continuity of $\Gamma$, it is not difficult to prove that 
\be\label{eq:barrondeg}
\mathcal{E}_M(f)\le c\left(\frac{\log M}{M}\right)^{1/2}|\mu|_{TV};
\ee
i.e., 
for any $f\in V$, \emph{there exists} $G\in\mathcal{G}_M$ such that
\be\label{eq:introbarron}
\|f-G\|\le c\left(\frac{\log M}{M}\right)^{1/2}|\mu|_{TV},
\ee
where $c$ is a positive constant independent of $M$ and $\mu$, but may depend upon $q, K,a$.
The formulation \eqref{eq:barrondeg} emphasizes the fact that  an explicit construction of $G$ satisfying \eqref{eq:introbarron} is not implied.

In the sequel, we use the notation $A\lesssim B$ to denote the fact that $A\le cB$ for some positive constant $c$ independent of the target function and $M$, but which could depend upon other fixed quantities of interest in the discussion. 
The notation $A\sim B$ denotes $A\lesssim B$ and $B\lesssim A$.

We make some observations here:
\begin{enumerate}
\item The error bound in \eqref{eq:introbarron} is independent of the dimension $q$, which makes it very attractive.
\item The class $V$ is difficult to describe using standard definition of smoothness, such as the number of derivatives, etc., although it may be possible to describe some conditions on mixed derivatives to ensure the membership in $V$ in this simple case.
\item Given data of the form $\{(\x_j,f(\x_j))\}_{j=1}^N$, it is tempting to find a network $G$ such that \eqref{eq:introbarron} is satisfied, for example, by solving an optimization problem of the form
\be\label{eq:regularization}
\mbox{ Minimize } \sum_{j=1}^N\left(f(\x_j)-\sum_{\k: |\k|_\infty\le M} a_\k\Gamma(\k\cdot \x_j -b_\k)\right)^2 +\lambda \sum_{\k: |\k|_\infty\le M} \left(|a_\k|+|\w_\k|_1+|b_\k|\right),
\ee
However, since the estimate \eqref{eq:introbarron} is not based on function evaluations, it is not clear that an
 empirical risk minimizer which imposes the additional constraint of having to  use such data (e.g., by solving the above minimization problem) will satisfy the error bound guaranteed by \eqref{eq:introbarron}.
 In fact, if the only information on $f$ is in terms of a small  number of its derivatives, the width results would imply that the bound \eqref{eq:introbarron} will not be satisfied.
 \item We have shown in \cite{mhaskar1997smooth} that if the activation function of a sequence of neural networks converging uniformly to  a function $f$ is smoother than $f$ in terms of number of derivatives then either the coefficients or the weights cannot be bounded; i.e., the regularization term in \eqref{eq:regularization} cannot work in this context.
\end{enumerate}

A totally constructive procedure (an example of the method referred to as pure approximation theory method in the introduction) for approximation from $\mathcal{G}_M$ is given in \cite{mhaskar1995degree, indiapap}. 
If $f\in W_{q,\gamma}$, then there exists $P=\sum_\k \hat{P}(\k)\exp(i\k\cdot\circ)\in\HH_n^q$ such that
\be\label{eq:introbest}
\|f-P\|\ls n^{-\gamma}\|f\|_{W_{q,\gamma}}.
\ee 
Explicit constructions based on data of the form $\{(\xi, f(\xi))\}$ for $\O(n^q)$ samples $\xi$ chosen randomly from the uniform distribution on $\TT^q$ are given in \cite{indiapap}. 
It is not difficult to verify that for $\x\in\TT^q$,
\be\label{eq:intro_poly_net}
P(\x)=\frac{1}{2\pi}\sum_\k \hat{P}(\k)\int_\TT \Gamma(\k\cdot\x-t)e^{it}dt.
\ee
We may discretize the integrals above using the trapezoidal rule and keep track of the errors using standard estimates to obtain for each $\k\in\ZZ^q$, $|\k|_2<n$,
a \textbf{pre-fabricated} network $G_\k\in \mathcal{G}_N$ such that
\be\label{eq:intropredesigned}
\left\|\frac{1}{2\pi}\int_\TT \Gamma(\k\cdot\x-t)e^{it}dt-G_\k\right\|\lesssim 1/N.
\ee 
Hence,  with $G=\sum_{\k}\hat{P}(\k)G_k$,
\be\label{eq:introPapprox}
\|P-G\| \ls \frac{\sum_\k|\hat{P}(\k)|}{N}.
\ee
Since $f\in W_{q,\gamma}$, our explicit constructions for $P$ show that $\sum_{\k\in\ZZ^q}|\hat{P}(\k)|^2|\k|_\infty^{2\gamma} \ls \|f\|_{q,\gamma}^2$.
Using Schwarz inequality, this  leads to 
\be\label{eq:introsum_of_coeff}
\sum_{\k}|\hat{P}(\k)|\ls \|f\|_{q,\gamma}\times
\begin{cases}
n^{q/2-\gamma}, &\mbox{ if $\gamma<q/2$},\\
\sqrt{\log n}, &\mbox{ if $\gamma=q/2$}\\
1, &\mbox{ if $\gamma>q/2$}.
\end{cases}
\ee
We now pick $N$ so that the upper bound in \eqref{eq:introPapprox} is $1/n^\gamma$.
 Together with \eqref{eq:intropredesigned} and \eqref{eq:introbest}, this leads to a network $G\in \mathcal{G}_M$ with $M\sim Nn^q$ such that 
\be\label{eq:introconst}
\mathcal{E}_M(f)\le \|f-G\|\ls \|f\|_{W_{q,\gamma}}\times
\begin{cases}
M^{-2\gamma/(3q)}, &\mbox{ if $\gamma<q/2$},\\
\disp\left(\frac{M}{\log M}\right)^{-\gamma/(	q+\gamma)}, &\mbox{ if $\gamma=q/2$}\\
M^{-\gamma/(q+\gamma)}, &\mbox{ if $\gamma>q/2$}.
\end{cases}.
\ee 
We note that this estimate is dimension dependent, but the network is obtained totally constructively using the data $\{(\xi, f(\xi))\}$. 
In this sense, it is stronger than the estimate \eqref{eq:barrondeg}.
There is no training required other than the solution of an underdetermined system of linear equations for obtaining the quadrature formulas defining the coefficients $\hat{P}(\k)$.
Thus, if the same points $\xi$ are used to sample different functions $f$, the network can be easily adapted by changing the coefficients of the prefabricated networks $G_\k$.
We remark that if a smooth version of the ReLU is chosen instead, as described in \cite{mhaskar2019analysis}, then the right hand side of the estimate \eqref{eq:intropredesigned} can be improved to $\exp(-cN)$ for  some positive constant $c$. 
The estimate \eqref{eq:introconst} then improves to 
\be\label{eq:introconstsmooth}
\|f-G\|\ls \left(\frac{M}{\log M}\right)^{-\gamma/q},
\ee
which is known be optimal in the sense of nonlinear widths \cite{devore1989optimal} , up to a logarithmic factor.
In this case, it can also be shown using ideas in \cite{mhaskar2019analysis} (cf. \cite{mhaskar1995degree}) that if $f\in W_{q,\gamma}$ for some $\gamma>m$, then for any derivative $\mathcal{U}$ of order $k\le m$, the \textbf{same network} that yields the bound \eqref{eq:introconstsmooth} also satisfies
\be\label{eq:introsimulapprox}
\|\mathcal{U}(f)-\mathcal{U}(G)\|\ls \left(\frac{M}{\log M}\right)^{-(\gamma-k)/q}.
\ee
Another work in this direction is \cite{siegel2023optimal}.

In \cite{mao2023rates}, the authors have considered an approach in between the purely probabilistic and purely approximation theory approaches. 
The idea is simple, namely to
 treat the expression \eqref{eq:intro_poly_net} as an expected value of the kernel $(\k,\x,t)\mapsto \Gamma(\k\cdot \x-t)$ with respect to an appropriate measure.
 The authors then use exceedingly difficult arguments to show that for $f\in W_{q,\gamma}$,  
\be\label{eq:intro_maobd}
\mathcal{E}_M(f)\ls \|f\|_{W_{q,\gamma}}\times
\begin{cases}
(\log M)^{(q+2)/2}M^{-(\gamma(q+2)/(q(q+4)))} &\mbox{ if $\gamma<q/2+2$},\\
(\log M)^{(q+3)/2}M^{-(\gamma(q+2)/(q(q+4)))} &\mbox{ if $\gamma=q/2+2$},\\
(\log M)^{1/2} M^{-(q+2)/(2q)}, &\mbox{if $\gamma>q/2+2$}.
\end{cases}
\ee
This is an improvement over both the bounds \eqref{eq:introbarron} and \eqref{eq:introconst} above. 
They deal with ``rough functions'' for which the pure probabilistic bounds are not valid, and the crude estimates for the pure approximation theory bounds give worse results for $q>2$.

We note another interesting aspect of these estimates. 
The pure (constructive) approximation theory estimates for approximation in $L^p$ norm require that the smoothness of the target function be expressed in terms of the same norm.
The pure probabilistic estimates depend upon the total variation of the measure $\mu$, which is analogous to the $L^1$ norm of a derivative (in an informal sense). 
As we will see in this paper, results similar to \eqref{eq:intro_maobd} can also be obtained for approximation in $L^p$ norms as well. 
If $p\ge 2$, the smoothness is measured in terms of the same $L^p$ norm.
However, since the argument rests on estimating the $L^2$ norm of the sequence replacing the sequence of Fourier coefficients, the smoothness needs to be measured in terms of the $L^2$ norm, even if the approximation is done in $L^p$, $1\le p<2$.

The arguments in \cite{mao2023rates} are exceedingly complicated partly because they force the more classical definition of smoothness classes from a Euclidean ball to the torus, use a more complicated form of the integral expression for the approximating polynomial $P$, and do not use the constructions given in \cite{indiapap} for constructing $P$ in a simpler manner.
The more refined arguments in this paper yield (cf. Theorem~\ref{theo:gtntheo} used only for function approximation, i.e., $a^*=a_*=0$, $\mathcal{U}_k$ replaced by identity) the estimates 
$$ \mathcal{E}_M(f)\ls
\|f|_{q,\gamma}\begin{cases}
\left(\frac{\log M}{M}\right)^{\gamma/q}&\mbox{ if $\gamma<q/2$},\\
\left(\frac{(\log M)^3}{M}\right)^{1/2}&\mbox{ if $\gamma=q/2$},\\
\left(\frac{\log M}{M}\right)^{1/2}&\mbox{ if $\gamma>q/2$}
\end{cases}
$$
in \eqref{eq:intro_maobd}. In the overlapping case $\gamma<q/2$, these are clearly better than those in \eqref{eq:intro_maobd}. 
We refer also to Remark~\ref{rem:maotong} for a different bound of the same nature for approximation by ReLU networks, where the bound in the case $\gamma>q/2$ is improved as well.

We note that the kernel considered in \cite{mao2023rates} has a formal expansion of the form
\be\label{eq:tongkerns}
\sum_{\gattha{j\in\ZZ}{j\not=0}} \frac{\exp(i\k\cdot \x)\exp(-ijt)}{j^2}.
\ee
Thus, we may view it as a sequence of  asymmetric kernels indexed by $\k\in\ZZ^q$.
In \cite{eignet}, we have generalized and sharpened the purely approximation theory approach to study a general kernel based approximation on arbitrary smooth compact manifolds, where the kernels involved are symmetric and have a Mercer expansion with the coefficients satisfying certain technical conditions.
In this context, it is not feasible to construct the moduli of smoothness required to define smoothness in general. 
One can define the smoothness in terms of a $K$-functional as in \cite{mauropap}, and obtain an equivalent characterization in terms of the degrees of approximation from the so-called classes of diffusion polynomials.
From the point of view of approximation theory, it is more natural to define the smoothness in terms of degrees of approximation, and then wonder (if desired) what other equivalent definitions can be given.
Another observation is that the coefficients of the analogue of $P$ do not characterize the smoothness of the target function even in the case of uniform approximation on $\TT^1$.
In \cite{mauropap, mhaskar2020kernel}, we have developed frames, where the norms of the individual terms do characterize the smoothness completely. 
One outcome of this  paper is a substantial improvement on the bounds \eqref{eq:intro_maobd}
(cf. Remark~\ref{rem:maotong}).

Given the other examples mentioned in Section~\ref{bhag:genintro}, 
 our aim is to generalize the results in \cite{mao2023rates} to a sequence of asymmetric kernels defined on general data spaces.
As a side benefit of the ideas, we generalize the results in \cite{bach2014} for approximation by ReLU$^r$ networks of functions in our smoothness classes.

\subsection{Outline of the paper}\label{bhag:outline}
 We review the relevant ideas about data spaces in Section~\ref{bhag:ddr}. 
The main ``recipe'' theorems (Theorems~\ref{theo:maintheo_n} and \ref{theo:maintheo})  are stated in Section~\ref{bhag:main}. 
The results obtained by applying these theorems in the special cases of general asymmetric kernels, twisted zonal function networks, and generalized translation networks are also described in Section~\ref{bhag:main}.
The theorems as they apply to ReLU$^r$ networks are given in Section~\ref{bhag:relu}. 
The proofs of all the theorems in Section~\ref{bhag:main} and ~\ref{bhag:relu} are given  in Section~\ref{bhag:proofs}.
The main contributions of the paper and further problems are commented upon in Section~\ref{bhag:conclusions}.
A list of symbols is given after Section~\ref{bhag:conclusions}.

\bhag{Data spaces}\label{bhag:ddr}
The purpose of this section is to review some background regarding data spaces. 
In Section~\ref{bhag:basicddr}, we review the basic definitions and notation.
Section~\ref{bhag:degapprox} introduces the important localized kernels and corresponding operators, and a fundamental theorem about the approximation properties of these operators.
In Section~\ref{bhag:smoothness}, we introduce the notion of smoothness of functions defined on a data space, and discuss the characterization of these spaces using certain localized frame operators.
The localization aspect of kernels and operators is not utilized fully in this paper, leaving this for future research.
In Section~\ref{bhag:asymeignets}, we enumerate the conditions on the asymmetric kernels, which we call asymmeric eignets, and illustrate the notion with three examples.

\subsection{Basic concepts}\label{bhag:basicddr}
We consider a compact metric measure space $\XX$ \yadi{$\XX$}{metric measure space}, with metric $\rho$ \yadi{$\rho$}{metric} and a probability measure $\mu^*=\mu^*_\XX$\yadi{$\mu^*$}{Distinguished probability measure, used with subscripts as needed}. 
We denote balls of $\XX$ by \yadi{$\BB(x,r)$}{Closed ball of radius $r$, centered $x$}
\be\label{eq:balldef}
\BB(x,r)=\{y\in \XX : \rho(x,y)\le r\},  \qquad x\in\XX, \ r>0.
\ee

We take $\{\lambda_k\}_{k=0}^\infty$ to be a non-decreasing sequence of real numbers with $\lambda_0=0$ and $\lambda_k\to\infty$ as $k\to\infty$, We allow repetititions in this sequence, but let $0=\hat{\lambda}_0<\hat{\lambda}_1 <\cdots$ be distinct values among the $\lambda_k$'s, arranged in increasing order. \yadi{$\lambda_k$, $\hat{\lambda}_\ell$}{Sequence defined in Section~\ref{bhag:ddr}, typically eigenvalues of the Laplace-Beltrami operator}
For integers $\ell\ge 0$, $ j,n\ge 1$, we denote
\be\label{eq:indexset}
 S_\ell=\{k:\lambda_k=\hat{\lambda}_\ell\}, \qquad S^*_n=\{k : \lambda_k<n\},\qquad\mathbf{S}_j=\{k : 2^{j-2}\le \lambda_k <2^j\}.  
\ee
\yadi{$S_\ell, S_n^*, \mathbf{S}_j$}{Index sets, cf. \eqref{eq:indexset}}
Next, let  $\{\phi_k\}_{k=0}^\infty$  be an orthonormal set in $L^2(\mu^*)$.
We assume that each $\phi_k$ is continuous. \yadi{$\phi_k$, $\psi_k$}{Orthonormal functions, typically eigen-funtions, cf. Section~\ref{bhag:ddr}}

Corresponding to the index sets defined in \eqref{eq:indexset}, we denote the spaces \yadi{$V_\ell, \Pi_n, \mathbf{V}_j$, $\Pi_\infty$}{Spaces of diffusion polynomials \eqref{eq:spacedef}}
\be\label{eq:spacedef}
V_\ell=\mathsf{span}\{\phi_k : k\in S_\ell\}, \qquad \Pi_n=\mathsf{span}\{\phi_k : k\in S^*_n\}, \qquad \mathbf{V}_j =\mathsf{span}\{\phi_k : k\in \mathbf{S}_j\}, \qquad \Pi_\infty=\bigcup_{n>0}\Pi_n.
\ee
The elements of the space $\Pi_n$
are called \emph{diffusion polynomials} (of order $<n$).

With this set up, the definition of a compact data space is the following.

\begin{definition}\label{def:ddrdef}
The  tuple $\Xi=(\XX,\rho,\mu^*, \{\lambda_k\}_{k=0}^\infty, \{\phi_k\}_{k=0}^\infty)$ is called a \textbf{(compact) data space} if 
each of the following conditions is satisfied. \yadi{$\Xi$}{Compact data space, Definition~\ref{def:ddrdef}}
\begin{enumerate}
\item $\XX$ is compact.
\item (\textbf{Ball measure condition}) There exist $q\ge 1$ and $\kappa>0$ with the following property: For each $x\in\XX$, $r>0$,
\be\label{eq:ballmeasurecond}
\mu^*(\mathbb{B}(x,r))=\mu^*\left(\{y\in\XX: \rho(x,y)<r\}\right)\le \kappa r^q.
\ee
(In particular, $\mu^*\left(\{y\in\XX: \rho(x,y)=r\}\right)=0$.)
\item (\textbf{Gaussian upper bound}) There exist $\kappa_1, \kappa_2>0$ such that for all $x, y\in\XX$, $0<t\le 1$,
\be\label{eq:gaussianbd}
\left|\sum_{k=0}^\infty \exp(-\lambda_k^2t)\phi_k(x)\phi_k(y)\right| \le \kappa_1t^{-q/2}\exp\left(-\kappa_2\frac{\rho(x,y)^2}{t}\right).
\ee
\end{enumerate}
We refer to $q$ as the \textbf{exponent} for $\Xi$.
With an abuse of terminology, we will refer to $\XX$ as the data space, the other notations being understood.
\end{definition}

\noindent\textbf{The constant convention.}
\emph{
In the sequel, $c, c_1,\cdots$ will denote generic positive constants depending only on the fixed quantities under discussion such as $\Xi$, $q$, $\kappa,\kappa_1,\kappa_2$,  the various smoothness parameters and the filters to be introduced. 
Their value may be different at different occurrences, even within a single formula.
The notation $A\ls B$ means $A\le cB$, $A\gs B$ means $B\ls A$ and $A\sim B$ means $A \ls B\ls A$.
\qed
}\\

It is shown in \cite[Proposition~5.1]{mhaskar2020kernel} that the Gaussian upper bound can be used to prove that
\be\label{eq:measureequivalence}
\mu^*(\mathbb{B}(x,r))\sim r^q, \qquad 0<r\le 1, \ x\in\XX.
\ee
We observe \cite[Lemma~5.2]{mhaskar2020kernel} that
\be\label{eq:phichristbd}
\sum_{k\in S_n^*}\phi_k(x)^2 \ls n^q, \qquad n \ge 1.
\ee
Consequently,
\be\label{eq:dimensionest}
|S^*_n|\ls n^q, \qquad  |\mathbf{S}_j|\ls 2^{jq}.
\ee
The primary example of a data space is, of course, a Riemannian manifold.

\begin{uda}\label{uda:manifold}
{\rm  Let $q\ge 1$ be an integer, $\XX$ be a smooth, compact, connected, finite dimensional Riemannian manifold (without boundary), $q$ be the dimension of $\XX$, $\rho$ be the geodesic distance on $\XX$, $\mu^*$ be the Riemannian volume measure normalized to be a probability measure, $\{\lambda_k^2\}$ be the sequence of eigenvalues of the (negative) Laplace-Beltrami operator on $\XX$, and $\phi_k$ be the eigenfunction corresponding to the eigenvalue $\lambda_k^2$; in particular, $\phi_0\equiv 1$. 
We have proved in  \cite[Appendix~A]{mhaskar2020kernel} that the Gaussian upper bound is satisfied. 
Therefore, if the condition in \eqref{eq:ballmeasurecond} is satisfied, then $(\XX,\rho,\mu^*, 
\{\lambda_k\}_{k=0}^\infty, \{\phi_k\}_{k=0}^\infty)$ is a data space with exponent equal to the dimension $q$ of the manifold. 
\qed}
\end{uda}
In  this paper, we will prove the recipe theorem for general data spaces, but illustrate it with  two special cases of Example~\ref{uda:manifold}. 
\begin{uda}\label{uda:specialmanifold_torus}
{\rm
Let  $\XX=\TT^q=\RR^q/(2\pi\ZZ)^q$\yadi{$\TT^q$}{torus of dimension $q$, Example~\ref{uda:specialmanifold_torus}} for different integer values of $q\ge 1$. 
$\mu^*=\mu^*_q$ in this case is the Lebsegue measure on $\XX$, normalized to be a  probability measure, $\rho(\x,\y)=|(\x-\y)\mbox{ mod } 2\pi|_2$. 
The system of orthogonal functions is $\{\exp(i\bs k \cdot \circ)\}_{\bs k \in \ZZ^q}$ with $\lambda_{\bs k}=|\bs k|_2$. 
In the notation of \eqref{eq:indexset} and \eqref{eq:spacedef}, $S_{\bs\ell}=\{\k : |\k|_2=|\bs\ell|_2\}$.
Of course, to be fastidious, we should use a judicious enumeration of $\ZZ^q$ and the sine and cosine functions instead, but it is easier to use $\ZZ^q$ itself as an indexing set and the exponential function, The details of the reduction to the properly enumerated real system are standard, but a bit tedious. 
\qed}
\end{uda}

\begin{uda}\label{uda:specialmanifold_sph}
{\rm
The other example is $\XX=\SS^q=\{\x\in\RR^{q+1}: |\x|_2=1\}$. \yadi{$\SS^q$}{Unit sphere embedded in $\RR^{q+1}$, Example~\ref{uda:specialmanifold_sph}}
The measure $\mu^*=\mu_q^*$ is the volume measure, normalized to be a probability measure, and $\rho$ is the geodesic distance on $\SS^q$.
It is well known that the eigenvalues of the (negative) Laplace-Beltrami operator are given by $\lambda_j^2=j(j+q-1)$ ($j\in\ZZ_+$) and the corresponding eigenspace is the space $\HH_j^q$ of all the homogeneous, harmonic $(q+1)$-variate polynomials of total degree $j$, restricted to $\SS^q$. \yadi{$\HH_j^q$}{Space of all homogeneous, harmonic spherical polynomials, Example~\ref{uda:specialmanifold_sph}}
The dimension of this space is denoted by $d_j^q$, and an orthogonal basis is denoted by $\{Y_{j,k}\}_{k=1}^{d_j^q}$. \yadi{$d_j^q$}{dimension of $\HH_j^q$}
\yadi{$\{Y_{\ell,k}\}_{k=1}^{d_j^q}$}{orthonormal basis for $\HH_j^q$}
It is known that $d_j^q\sim j^{q-1}$.
In this context, we denote $\Pi_n$ by $\Pi_n^q$ \yadi{$\Pi_n^q$}{space of spherical polynomials of degree $<n$} to emphasize that we are working on $\SS^q$. The space $\Pi_n^q$ comprises restrictions to $\SS^q$ of $(q+1)$-variable algebraic polynomials of degree $<n$.
In the notation of \eqref{eq:indexset} and \eqref{eq:spacedef}, $S_\ell=\{m: m(m+q-1)=\ell(\ell+q-1)\}$, $V_\ell=\{Y_{\ell,m}: m=1,\cdots, d_\ell^q\}$ and $V_n^*=\Pi_n^q$.

The addition formula (cf.  \cite{mullerbk} and \cite[Chapter XI, Theorem 4]{batemanvol2}) states that
\be\label{eq:addformula}
\sum_{k=1}^{d_j^q} Y_{\ell,k}(\x)\overline{Y_{\ell,k}(\y)} =
\frac{\omega_q}{\omega_{q-1}} p_j^{(q/2-1,q/2-1)}(1)p_j^{(q/2-1,q/2-1)}(\x\cdot\y), \qquad
j=0,1,\cdots, 
\ee
where
\be\label{eq:sphvolume}
\omega_q=\frac{2\pi^{(q+1)/2}}{\Gamma((q+1)/2)}
\ee
is the Riemannian volume of $\SS^q$, and $p_j^{(q/2-1,q/2-1)}$ is the orthonormalized ultraspherical polynomial satisfying 
\be\label{eq:ultraortho}
\int_{-1}^1 p_j^{(q/2-1,q/2-1)}(t)p_\ell^{(q/2-1,q/2-1)}(t)(1-t^2)^{q/2-1}dt =\delta_{j,\ell}, \qquad j,\ell=0,1,\cdots.
\ee
\yadi{$\omega_q$}{volume of $\SS^q$}
\yadi{$p_j^{(q/2-1,q/2-1)}$}{orthonormalized ultraspherical polynomials cf. \eqref{eq:ultraortho}}
We need not go into the details of the construction of an orthonormal basis of polynomials in each $\HH_j^q$, but consider an enumeration $\{\phi_k\}$ of the orthonormal basis for $\disp\oplus_{j=0}^\infty \HH_j^q$ so that polynomials of lower degree appear first in this enumeration.

\qed}
\end{uda}
\begin{remark}\label{rem:graph}
{\rm
In \cite{friedman2004wave}, Friedman and Tillich give a construction for an orthonormal system on a graph which leads to a finite speed of wave propagation. 
It is shown in \cite{frankbern} that this, in turn, implies the Gaussian upper bound. 
Therefore, it is an interesting question whether appropriate definitions of measures and distances can be defined on a graph to satisfy the assumptions of a data space.
\qed}
\end{remark}

\subsection{Degree of approximation}\label{bhag:degapprox}
Let $1\le p\le\infty$. For $\mu^*$-measurable $A\subset \XX$ and $f:A\to\RR$, we define
\be\label{eq:lpdef}
\|f\|_{p;A}=\begin{cases}
\disp \left\{\int_A |f(x)|^pd\mu^*(x)\right\}^{1/p}, &\mbox{ if $1\le p<\infty$,}\\
\disp \esssup_{x\in A}|f(x)|, & \mbox{if $p=\infty$}.
\end{cases}
\ee
The space $L^p(A)$ comprises functions $f$ for which $\|f\|_{p;A}<\infty$, with the convention that two functions are identified if they are equal $\mu^*$-almost everywhere.
The space $C(A)$ comprises uniformly continuous and bounded real valued functions on $A$.
When $A=\XX$, we omit its mention from the notation. 
Thus, 
we write $\|\cdot\|_p$ in place of $\|\cdot\|_{p;\XX}$ and $L^p$ in  place of $L^p(\XX)$.

For $f\in L^p$, $n>0$, we define the \emph{degree of approximation} to $f$ by \yadi{$E_{p;n}$}{Degree of approximation, \eqref{eq:degapproxdef}}
\be\label{eq:degapproxdef}
E_{p;n}(f)=\min_{P\in\Pi_n}\|f-P\|_p.
\ee
The space $X^p$ comprises functions $f$ for which $E_{p;n}(f)\to 0$ as $n\to\infty$.
We will assume that $\Pi_\infty$ is dense in $C(\XX)$, so that $X^\infty=C(\XX)$ and  (hence) $X^p=L^p(\XX)$ if $1\le p<\infty$.
\yadi{$X^p$}{$L^p$-closure of $\Pi_\infty$}
In this section, we describe certain localized kernels and operators.
The localization property itself is not utilized fully in this paper, but we need the fact that the operators yield ``good approximation'' in the sense of Theorem~\ref{theo:goodapprox} below.

The kernels are defined by
\be\label{eq:kerndef}
\Phi_{n}(H;x,y)= \sum_{k=0}^\infty H\left(\frac{\lambda_k}{n}\right)\phi_k(x)\phi_k(y),
\ee
where $H :\RR\to\RR$ is a compactly supported function. \yadi{$H$}{Band pass filter, Section~\ref{bhag:degapprox}}\yadi{$\Phi_n$}{diffusion polynomial kernels, \eqref{eq:kerndef}}

The  operators corresponding to the kernels $\Phi_n$ are defined by
\be\label{eq:opdef}
\sigma_{n}(H;f)(x)= \int_{\mathbb{X}}\Phi_{n}(H;x,y)f(y)d\mu^*(y) =\sum_{k}H\left(\frac{\lambda_k}{n}\right)\hat{f}(k)\phi_k(x),
\ee
where
\be\label{eq:fourcoeff}
\hat{f}(k)=\int_\XX f(y)\phi_k(y)d\mu^*(y).
\ee
\yadi{$\sigma_n, \tau_j$}{Reconstruction and analysis operators, \eqref{eq:opdef}, \eqref{eq:analopdef}}
The following proposition recalls an important property of these kernels. 
Proposition~\ref{prop:kernloc}  is proved in \cite{mauropap}, and more recently  in much greater generality in \cite[Theorem~4.3]{tauberian}.
\begin{prop}\label{prop:kernloc}
Let  $S>q+1$ be an integer, $H:\mathbb{R}\to \mathbb{R}$ be an even, $S$ times continuously differentiable, compactly supported function.
 Then for every $x,y\in \mathbb{X}$, $N\ge 1$,
\begin{equation}\label{eq:kernlocest}
| \Phi_N(H;x,y)|\ls \frac{N^{q}}{\max(1, (N\rho(x,y))^S)},
\end{equation}
where the constant may depend upon $H$ and $S$, but not on $N$, $x$, or $y$.
\end{prop}

In the remainder of this paper, we fix a filter $H$; i.e., an infinitely differentiable function $H: [0,\infty)\to [0,1]$, such that $H(t)=1$ for $0\le t\le 1/2$, $H(t)=0$ for $t\ge 1$. 
The domain of the filter $H$ can be extended to $\RR$ by setting $H(-t)=H(t)$.
The filter $H$ being fixed, its mention will be omitted from the notation.

The following theorem gives a crucial property of the operators, proved in several papers of ours in different contexts, see \cite{mhaskar2020kernel} for a recent proof.

\begin{theorem}\label{theo:goodapprox}
Let $n>0$. If $P\in\Pi_{n/2}$, then $\sigma_n(P)=P$. Also, for $1\le p\le\infty$, 
\be\label{opbd}
\|\sigma_n(f)\|_p \ls \|f\|_p, \qquad f\in L^p(\XX).
\ee 
If  $f\in L^p$, then
\be\label{goodapprox}
E_{p;n}(f)\le \|f-\sigma_n(f)\|_p\ls E_{p;n/2}(f).
\ee
\end{theorem}

\subsection{Smoothness classes}\label{bhag:smoothness}
Our goal is to approximate a target function in a smoothness class by  networks based on the activation function $G$.
We define two versions of smoothness; one for the activation function $G$, and the other for the class of target functions.

The smoothness of the activation function is needed for going from pointwise bounds to uniform bounds in the use of H\"offding's inequality. This is just the usual H\"older continuity.
\begin{definition}\label{def:Gsmooth}
Let $0 < \alpha\le 1$. The class $\mathsf{Lip}(\alpha)$ comprises $f :\XX\to \RR$ for which
\be\label{eq:lipnormdef}
  \|f\|_{\mathsf{Lip}(\alpha)}=\|f\|_\infty+\sup_{x\not= x'\in \XX}\frac{|f(x)-f(x')|}{\rho(x,x')^\alpha}<\infty.
\ee
\end{definition}

The smoothness classes of the target function need to be defined in a more sophisticated manner. 
From an approximation theory perspective, this is done best in terms of the degrees of approximation.

\begin{definition}\label{def:targetsmooth}
Let $\gamma>0$, $1\le p\le\infty$. We define the (Sobolev) class $W_{p;\gamma}=W_{p;\gamma}(\XX)$ as the space of all $f\in X^p$ for which
\be\label{eq:targetsmooth}
\|f\|_{W_{p;\gamma}}=\|f\|_p+\sup_{j\ge 0}2^{j\gamma}E_{p;2^j}(f)<\infty.
\ee
\end{definition}
Thus, $W_{p;\gamma}(\XX)$ \yadi{$W_{p;\gamma}(\XX)$}{Sobolev approximation space \eqref{eq:targetsmooth}} is the class of functions for which $E_{n,p}\ls n^{-\gamma}$. 
\begin{remark}\label{rem:besov}
{\rm
Characterizations of the spaces $W_{p;\gamma}$ in terms of derivatives and their Lipschitz/H\"older continuity (in the sense of $L^p$, cf. \eqref{eq:modcont}) are known for some manifolds $\XX$, such as the torus $\TT^q$ (cf. Section~\ref{bhag:techintro}).
There are many definitions of Sobolev spaces, typically in the context of Euclidean domains, e.g. \cite{adams2003sobolev, nikolskii}.
The term Sobolev space is sometimes reserved for the case when $\gamma$ is an integer, with the extension to non-integer $\gamma$ given different names.
In \cite{adams2003sobolev}, these are defined in terms of intermediate spaces, and the discussion is implict in the discussion of Besov spaces.
 In  \cite{nikolskii},  these are denoted by $H^\gamma_p$, and equivalence theorems in terms of degree of approximation by entire functions (trigonometric polynomials in the periodic case) are given.
 The book \cite{adams2003sobolev} gives characterizations in terms of wavelet coefficients, similar to Theorem~\ref{theo:equivtheo} below. 
 All these classical definitions require special structures  of the Euclidean spaces. 
 The advantage of defining these classes in terms of degrees of approximation as we have done is that they hold in a broad context.
In the case of general manifolds, this can be described in terms of $K$-functionals.
We do not need to use this information in our paper. \qed}
\end{remark}

We describe now a characterization of the smoothness classes $W_{p;\gamma}$ in terms of our operators.
We define the \textit{analysis operators} $\tau_j$ as follows.
\be\label{eq:analopdef}
\tau_j(f)=\begin{cases}
\sigma_1(f), &\mbox{ if $j=0$,}\\
\sigma_{2^j}(f)-\sigma_{2^{j-1}}(f), &\mbox{ if $j\ge 1$.}
\end{cases}
\ee
Theorem~\ref{theo:goodapprox} implies that for every $f\in X^p$,
\be\label{eq:paleywiener}
f=\sum_{j=0}^\infty \tau_j(f),
\ee
and
\be\label{eq:remainder}
\sigma_{2^n}(f)=\sum_{j=0}^n \tau_j(f), \qquad f-\sigma_{2^n}(f)=\sum_{j=n+1}^\infty \tau_j(f),
\ee
with all the infinite series converging in the sense of $L^p$.  
The following theorem is not difficult to prove using Theorem~\ref{theo:goodapprox}, and \eqref{eq:remainder}.
For part (c), we note that the space $W_{2;\gamma}$ is the same as a certain Besov space. 
We refer to  \cite[Proposition~2]{besovquadpap} where this is worked out in detail in the case when $\XX$ is a sphere.
The same arguments work in the general case.

\begin{theorem}
\label{theo:equivtheo}
Let  $1\le p\le\infty$, $f\in X^p$. \\
{\rm (a)} We have, with convergence in the sense of $X^p$,
\be\label{eq:lpexp}
f=\sum_{j=0}^\infty \tau_j(f).
\ee
{\rm (b)} Let $\gamma>0$. Then
\be\label{eq:equivrel}
\|f\|_{W_{p;\gamma}}\sim \|f\|_p+\sup_{j\ge 0}2^{j\gamma}\|\tau_j(f)\|_p.
\ee
{\rm (c)} Let $\gamma>0$, and $p=2$. Then
\be\label{eq:l2equivrel}
\|f\|_2^2 \sim \sum_{j=0}^\infty \|\tau_j(f)\|_2^2, \qquad
\|f\|_{W_{2;\gamma}}^2\sim \|f\|_2^2+\sum_{j=0}^\infty 2^{2j\gamma}\|\tau_j(f)\|_2^2.
\ee
\end{theorem}

We note the Nikolskii inequalities (cf. \cite[Proposition~5.4]{mhaskar2020kernel}): If $1\le p<r\le\infty$, 
\be\label{eq:nikolskii}
\|P\|_p \le \|P\|_r \ls n^{q(1/p-1/r)}\|P\|_p, \qquad P\in \Pi_n.
\ee
Using these and Theorem~\ref{theo:equivtheo}, it is easy to verify that if $1\le p<r\le \infty$,  $\gamma>0$,  then
\be\label{eq:sobolev_class_embedding}
W_{p;\gamma+q(1/p-1/r)}\subseteq W_{r;\gamma} \subseteq W_{p;\gamma}
\ee
in the sense of continuous embedding. 
In particular,
\be\label{eq:sobolevembedding}
\begin{cases}
 \|f\|_{W_{p;\gamma}}\ls \|f\|_{W_{r;\gamma}}, &\mbox{if $f\in W_{r;\gamma}$},\\
 \|f\|_{W_{r;\gamma}} \ls \|f\|_{W_{p;\gamma+q(1/p-1/r)}}, &\mbox{if $f\in W_{p;\gamma+q(1/p-1/r)}$}.
 \end{cases}
\ee

\begin{definition}\label{def:deroperator}
Let $\XX$ be a data space,  $1\le p\le \infty$. 
 A linear operator $\mathcal{U}$ defined on $\Pi_\infty$ (and extended to a subspace of $X^p$)  is called \textbf{derivative-like (with exponent $a\in\RR$)} if $\mathcal{U}$ is closed in $X^p$, and satisfies (cf. \eqref{eq:spacedef})
\be\label{eq:fullbernstein}
 \|\mathcal{U}(P)\|_p\ls 2^{ja}\|P\|_p, \qquad P\in \mathbf{V}_{2^j},\ j=0,1,\cdots.
 \ee
 The constants involved may depend upon $\mathcal{U}$ and $p$ as well. \yadi{$\mathcal{U}$}{derivative-like operator, Section~\ref{bhag:smoothness}}
 \end{definition}
 
\begin{uda}\label{uda:multiplier}
{\rm
A simple example of a derivative-like operator is the identity operator $P\mapsto P$, and more generally, multiplication by a continuous function; $P\mapsto \phi P$ for some $\phi\in C(\XX)$. 
Clearly, the exponent is $0$.
\qed}
\end{uda}
\begin{uda}\label{uda:deroperator}
{\rm
A \textbf{pseudo-differential operator} $\mathcal{U}$ on $X^p$ is defined spectrally by $\widehat{\mathcal{U}(f)}(\ell)=b(\lambda_\ell)\hat{f}(\ell)$ (cf. \eqref{eq:fourcoeff}) for some function $b: [0,\infty)\to\RR$. 
Under certain conditions on $b$, intuitively that $b(\lambda_k)\sim \lambda_k^a$, it can be shown as in \cite{eignet} that $\mathcal{U}$ is derivative-like of order $a$.
In this case, negative values of $a$ are allowed.
\qed}
\end{uda}

\begin{uda}\label{uda:derivativeop}
{\rm 
In the case when $\XX$ is a manifold as in Example~\ref{uda:manifold}, it is possible to define a derivative of an integer order $a>0$  as an operator on $\Pi_\infty$. This operator is not necessarily a pseudo-differential operator, but it is shown in \cite{frankbern} that it satisfies  \eqref{eq:fullbernstein}. 
So, it is a derivative-like operator  in the sense of Definition~\ref{def:deroperator}.
More generally, a linear differential operator with smooth coefficients is derivative-like.
\qed}
\end{uda}
The following proposition is easy to deduce using Theorem~\ref{theo:equivtheo}.
\begin {prop}\label{prop:derivative}
Let $1\le p\le\infty$, $f\in X^p$, $\mathcal{U}$ be a derivative-like operator with exponent $a$. \\
{\rm (a)} We have
\be\label{eq:ufapprox}
\|\mathcal{U}(f)-\mathcal{U}(\sigma_{2^n}(f))\|_p\ls \sum_{j=n+1}^\infty 2^{ja}\|\tau_j(f)\|_p.
\ee
{\rm (b)} In particular, if $\gamma>a$ and $f\in W_{p;\gamma}$,  then $\mathcal{U}(f) \in W_{p;\gamma-a}$.
\end{prop}

\subsection{Asymmetric eignets}\label{bhag:asymeignets}
In the sequel, we assume $\XX$ to be a compact data space with exponent $q$, and $\YY$ to be a measure space, equipped with a probability measure $\mu_\YY^*$.
The following definition is motivated by the example in Section~\ref{bhag:techintro}, equation \eqref{eq:tongkerns} in particular. 
Additional examples are given after the definition. 

\begin{definition}\label{def:asymetric_eignet}
Let $\alpha>0$, $\beta\in\RR$. 
An \textbf{asymmetric eignet kernel (with exponents $(\alpha,\beta)$)} is a function $G :\ZZ_+\times\XX\times\YY\to\RR$, satisfying each of the following properties:
\begin{enumerate}

\item (\textbf{Connection condition}) There exist $\mu^*_\YY$-measurable functions $\mathcal{D}_G\phi_\ell :\YY\to \RR$, $\ell\in \ZZ_+$ such that for $\ell\in\ZZ_+$ and $x\in\XX$, we have
\be\label{eq:basicass}
\phi_\ell(x)=\int_\YY G(\ell;x,y)\mathcal{D}_G\phi_\ell(y)d\mu^*_{\YY} (y), \qquad x\in\XX.
\ee
Moreover, (cf. \eqref{eq:indexset})
\be\label{eq:invchristbd}
\int_\YY |\mathcal{D}_G\phi_\ell(y)|^2d\mu_\YY^*(y) \ls 2^{2j\beta}, \qquad \ell\in \mathbf{S}_j, \ j\in\ZZ_+.
\ee

\item (\textbf{Smoothness condition})
There exists $\alpha\in (0,1]$ such that for every $\ell\in\ZZ_+$,
\be\label{eq:lipcond}
\sup_{ y\in\YY}\|G(\ell; \cdot, y)\|_{\mathsf{Lip}(\alpha)} <\infty.
\ee
\end{enumerate}
An asymmetric eignet is a function of the form $x\mapsto \sum_{j=1}^{n} a_j G(\ell_j;x,y_j)$, $x\in\XX$, $y_1,\cdots,y_n\in\YY$, $\ell_1,\cdots,\ell_n\in\ZZ_+$, and $a_1,\cdots, a_n\in\RR$ (or $\CC$ as appropriate).
\end{definition}
\yadi{$G(\ell,x,y)$}{Asymmetric eignet kernel}
\yadi{$\mathcal{D}_G$}{special operator associated with $G$, \eqref{eq:basicass}, \eqref{eq:gendfdef}}
If $P\in \Pi_n$, we may use \eqref{eq:basicass} to define
\be\label{eq:polydgdef}
\mathcal{D}_G(P)(y)=\sum_{\ell=0}^{|S_n^*|} \hat{P}(\ell)\mathcal{D}_G\phi_\ell(y), \qquad y\in\YY.
\ee
We may extend this definition to $X^p$ formally by
\be\label{eq:gendfdef}
\mathcal{D}_G(f)(y)=\sum_\ell \hat{f}(\ell)\mathcal{D}_G\phi_\ell(y), \qquad y\in\YY.
\ee
Thus, there is a formal relationship reminiscent of the variation (or native) space relationship:
\be\label{eq:ass_variation}
f(x)=\int_{\ZZ_+\times \YY}G(\ell;x,y)d\nu_G(f)(\ell,y),
\ee
where, with $\mathfrak{c}$ being the counting measure on $\ZZ_+$ (that associates the mass $1$ with each integer), \yadi{$\nu_G$}{Measure associated with eignents, \eqref{eq:ass_measure}}
\be\label{eq:ass_measure}
d\nu_G(f)(\ell,y)=\hat{f}(\ell)\mathcal{D}_G\phi_\ell(y)d\mu_\YY^*(y)d\mathfrak{c}(\ell)
\ee
We note that the operator $f\mapsto \nu_G(f)$ is a linear operator. \yadi{$\mathfrak{c}$}{counting measure on $\ZZ_+$}

We enumerate a few examples which have motivated our work. 
We will define the asymmetric eignet kernels, the corresponding network is defined analogous to eignets; e.g., a \textbf{twisted zonal function network} (cf. Example~\ref{uda:sbf} below)  is a mapping of the form 
$$
\x\mapsto \sum_{j=1}^n \sum_{\ell=1}^m w_{j,\ell}G(\x\cdot R_\ell \y_j), \qquad \x\in\SS^q,\ R_\ell\in SO(q+1), \ w_{j,\ell}\in\RR,
$$ 
where $G$ is a zonal function.
\begin{uda}\label{uda:randomfeature}
{\rm (\textbf{General SVD kernels})
In this example, we consider a (single) general non-symmetric kernel on $\XX\times\YY$ which admits a singular value decomposition of the form
  \be\label{eq:klexpansion}
 G(x,y)\sim \sum_{k=0}^\infty \frac{\phi_k(x)\psi_k(y)}{\Lambda_k},
  \ee
   where $\{\psi_k\}\subset L^2(\mu^*_{\YY})$ (respectively, $\{\phi_k\}\subset L^2(\mu^*_{\XX})$) is an orthonormal set of functions, and $\Lambda_k>0$.

 One important example is a \textit{random process} on a data space $\XX$ with the random variable taken from a probability distribution $\mu^*_{\YY}$ on a measure space $\YY$ is a kernel $G:\XX\times\YY\to \RR$. The singular value decomposition in this case is called the \textit{Karhunen-Lo\'eve expansion}.
 We may think of $G(x,y)$ as a random feature of $x$.

It is clear that the kernel $G$ satisfies the connection condition \eqref{eq:basicass} with 
\be\label{eq:random_connection}
\mathcal{D}_G\phi_k (y)=\Lambda_k\psi_k(y), \qquad k=0,1,\cdots.
\ee
If $\Lambda_k\ls k^\beta$ for some $\beta>0$, then
we have \eqref{eq:invchristbd}.
\qed
}
\end{uda}
\begin{uda}\label{uda:sbf}
{\rm(\textbf{Twisted zonal function})
 In this example, we take $\XX=\YY=\SS^q$, $\mu_q^*$ to be the volume measure on $\SS^q$, normalized to be a probability measure.
 We continue the notation introduced in Example~\ref{uda:specialmanifold_sph}.
 In this example again, there is essentially only one asymmetric kernel involved: $G(\x\cdot R\y)$ for some function $G$ as described below, and a rotation $R\in SO(q+1)$, although our analysis allows us to consider a finite number of kernels of this form for different rotations.
 
 Let $G:[-1,1]\to\RR$ be square integrable with respect to the measure $(1-t^2)^{q/2-1}dt$ on $[-1,1]$. 
 Then $G$ has a formal expansion of the form
\be\label{eq:Gorthoexpansion}
G(t)\sim \frac{\omega_q}{\omega_{q-1}}\sum_{j=0}^\infty \hat{G}(j)p_j^{(q/2-1,q/2-1)}(1)p_j^{(q/2-1,q/2-1)}(t), \qquad t\in [-1,1].
\ee
A \textbf{zonal function} is a kernel of the form $(\x,\y)\mapsto G(\x\cdot\y)$.
The addition formula \eqref{eq:addformula}  leads to a formal expansion of the form
\be\label{eq:Gzonal_expansion}
G(\x\cdot\y)=\sum_{j=0}^\infty \hat{G}(j)\sum_{k=1}^{d_j^q}Y_{j,k}(\x)\overline{Y_{j,k}(\y)}, \qquad \x,\y\in\SS^q.
\ee
To express this expansion in terms of 
 the notation of Example~\ref{uda:specialmanifold_sph}, we introduce the notation that for any $t\in\RR$,
$$
t^{[-1]}=\begin{cases}
1/t, &\mbox{ if $t\not=0$,}\\
 0, &\mbox{if $t=0$.}
 \end{cases}
$$
The formula \eqref{eq:Gzonal_expansion} can be written in the form
\be\label{eq:Gzonal_expansion_bis}
G(\x\cdot\y)=\sum_{\ell=0}^\infty \frac{\phi_\ell(\x)\overline{\phi_\ell(\y)}}{\Lambda_\ell},
\ee
where 
\be\label{eq:zonaleig}
\Lambda_\ell=(\hat{G}(j))^{[-1]}, \qquad \ell\in S_j, \ j\in\ZZ_+.
\ee
We note that $\Lambda_\ell$ is different from $\lambda_\ell$ as described in Example~\ref{uda:specialmanifold_sph}. 
We will say that $G$ is of type $\beta$ if 
\be\label{eq:zonaltype}
|\Lambda_\ell|\ls 2^{j\beta}, \qquad \ell\in S_j, \ j\in\ZZ_+.
\ee
Let $SO(q+1)$ be the space of all rotations of $\RR^{q+1}$ with determinant $=1$.
 A \textbf{twisted zonal function} is a kernel of the form $(\x,\y)\mapsto G(\x\cdot R\y)$, where $G$ is  a zonal function, and $R\in SO(q+1)$. 
 It is well known (e.g., \cite[Chapter~9]{vilenkin1978special}) that there are functions $t_{k,\ell}\in C(SO(q+1))$, orthonormalized with respect to the Haar measure on the rotation group $SO(q+1)$, such that if $\phi_\ell \in \HH_L^q$, then
\be\label{eq:gprepresent}
\begin{aligned}
\phi_\ell(\x)&=\sum_{k\in S_L}t_{k,\ell}(R)\phi_k(R^{-1}\x)\\
&=\sum_{k\in S_L}t_{k,\ell}(R)\Lambda_k \int_{\SS^q}G(R^{-1}\x\cdot\y)\phi_k(\y)d\mu_q^*(\y)\\
&=\sum_{k\in S_L}t_{k,\ell}(R)\Lambda_k\int_{\SS^q} G(\x\cdot R\y)\phi_k(\y)d\mu_q^*(\y).
\end{aligned}
\ee
It is easy to verify  that for any $R\in SO(q+1)$ and $L\ge 1$, the matrix $[t_{k,\ell}(R)]_{k,\ell\in S_L}$ is an orthogonal matrix as well. \yadi{$t_{k,\ell}$}{Orthonormal functions on $SO(q+1)$, cf. \eqref{eq:gprepresent}}
We extend the matrix $[t_{k,\ell}(R)]$ setting the entries to be $0$ when $\lambda_k\not=\lambda_\ell$ (recall the notation in Example~\ref{uda:specialmanifold_sph}).
It is then easy to verify that a twisted zonal function of type $\beta$ satisfies \eqref{eq:basicass} with 
\be\label{eq:sphdg}
\mathcal{D}_G\phi_\ell=\sum_{k\in S_L}t_{k,\ell}(R)\Lambda_k\phi_k,
\ee
 and \eqref{eq:invchristbd} with $\beta$ as in \eqref{eq:zonaltype}.
\qed
}
\end{uda}
\begin{uda}\label{uda:gentransnet}
{\rm (\textbf{Generalized translations})
Let $q\ge d\ge 1$ be integers, $\YY=\TT^d$, $\XX=\TT^q$,   $\mu_d^*$ (respectively, $\mu_q^*$) be the Lebesgue measure on $\TT^d$ (respectively, $\TT^q$), normalized to be a probability measure, 
 and for $\bs \ell\in\ZZ^q$,
 $$
 \phi_{\bs{\ell}}(\x)=\exp(i\bs{\ell}\cdot \x).
 $$ 
 Following \cite{mhaskar1995degree}, a \textit{generalized translation network} is defined as a mapping
 \be\label{eq:gentransdef}
 \mathbb{G}_{GTN}(\x)=\sum_{j,k}w_{j,k}G(A_k \x-\y_{j,k}),
\ee
 where $G\in C(\TT^d)$, $w_{j,k}\in \RR$, $\y_{j,k}\in\TT^d$, and $A_k$'s are $d\times q$ matrices with integer entries. 
 We note that this is a sequence of asymmetric eignet kernels $(k,\x,\y)\mapsto G(A_k\x-\y)$, $\x\in\TT^q$, $\y\in\TT^d$.
 Thus,  neural networks are generalized translation networks with $d=1$. When $d=q$, $G$ is a radial function, and $A_\ell$'s are identity matrices, we obtain a radial basis function network.
 
 We will use the notation $\bs{1}=(1,\cdots,1)^T\in \ZZ^d$.
 It is easy to construct a matrix $A_{\bs{\ell}}$ such that
 $$
\bs{\ell}=A_{\bs{\ell}}^T\bs{1}.
$$
For example, we may stack $d\times q$ matrices with successive entries on $\bs{\ell}$ on the diagonal. 
For example, if $d=3$, $q=8$, 
$$
A_{\bs{\ell}}=\begin{pmatrix}\ell_1 & 0& 0& \ell_4&0&0&\ell_7&0\\
0&\ell_2&0&0&\ell_5&0&0&\ell_8\\
0&0&\ell_3&0&0&\ell_6&0&0

\end{pmatrix}
$$
Then for any $\x\in\TT^q$,
$$
\exp(i\bs{\ell}\cdot\x)=\exp(iA_{\bs{\ell}}^T\bs{1}\cdot \x)=\exp(i\bs{1}\cdot A_{\bs{\ell}}\x).
$$
If we assume  that $\hat{G}(\bs{1})\not=0$, then
$$
1=\frac{1}{\hat{G}(\bs{1})}\int_{\TT^d}G(\y)\exp(-i\bs{1}\cdot\y)d\mu_d^*(\y),
$$
so that for $\bs\ell\in\ZZ^q$, $\x\in\TT^q$,
$$
\begin{aligned}
\phi_{\bs{\ell}}(\x)&=\exp(i\bs{\ell}\cdot\x)=\frac{1}{\hat{G}(\bs{1})}\int_{\TT^d}G(\y)\exp(i(\bs{\ell}\cdot\x-\bs{1}\cdot\y)))d\mu_d^*(\y)\\
&=\frac{1}{\hat{G}(\bs{1})}\int_{\TT^q}G(\y)\exp(i(\bs{1}\cdot (A_{\bs{\ell}}\x-\y)))d\mu_d^*(\y)\\
&=\frac{1}{\hat{G}(\bs{1})}\int_{\TT^q}G(A_{\bs{\ell}}\x-\y)\exp(i(\bs{1}\cdot\y))d\mu_d^*(\y)
 \end{aligned}
 $$
It is clear that the kernel $(\bs\ell,\x,\y)\mapsto G(A_{\bs\ell}\x-\y) $  satisfies the connection condition of Definition~\ref{def:asymetric_eignet} with 
\be\label{eq:gentransnet_dg}
\mathcal{D}_G\exp(i\bs{\ell}\cdot\circ)(\y)= \frac{\exp(i(\bs{1}\cdot\y))}{\hat{G}(\bs{1})}, \qquad \bs{\ell}\in\ZZ^q.
\ee 
Thus, \eqref{eq:invchristbd} is satisfied with $\beta=0$.
In  the Definition~\ref{def:asymetric_eignet}, we defined the sequence of kernels using $\ZZ_+$ as the index set. The notation is reconciled by using  a judicious enumeration of $\ZZ^q$.
\qed
}
\end{uda}

\bhag{Main results}\label{bhag:main}
Our main theorem is a ``recipe theorem'' about the degree of approximation by asymmetric eignets in general. 
This will be applied to obtain concrete theorems for the cases of generalized translation networks, twisted zonal function networks, and general SVD kernel networks.

\begin{definition}\label{def:compatible}
Let $G$ be an asymmetric eignet kernel. 
A linear operator $\mathcal{U}$ on $\Pi_\infty$ is \textbf{compatible with $G$} if $\mathcal{U}(G(\ell;\circ,y))$ is well defined and measurable on $\ZZ_+\times \YY$ and
\be\label{eq:compatible}
\mathcal{U}(\phi_\ell)(x)=\int_\YY \mathcal{U}(G(\ell;\circ,y))(x)\mathcal{D}_G\phi_\ell(y)d\mu^*_{\YY} (y), \qquad x\in\XX.
\ee
We will abbreviate $\mathcal{U}(G(\ell;\circ,y))(x)$ by $\mathcal{U}(G)(\ell;x,y)$.
\end{definition}

\begin{theorem}
\label{theo:maintheo_n}
Let $0<\alpha\le 1$, $\beta>0$, $G$ be an asymmetric eignet kernel with exponents $(\alpha,\beta)$,   and $\{\mathcal{U}_1,\cdots, \mathcal{U}_J\}$ be a set of  operators, each compatible with $G$ and derivative-like, with $a_k$ being the exponent for $\mathcal{U}_k$, $k=1,\cdots, J$. 
Let
\be\label{eq:minexponent}
a_*=\min_{1\le k\le J}a_k,\qquad a^*=\max_{1\le k\le J}a_k
\ee
\yadi{$a^*$, $a_*$}{maximum and minimum orders of derivative-like operators, \eqref{eq:minexponent}}
For integer $n\ge 0$, we define \yadi{$T_n$}{\eqref{eq:tndef}}
\be\label{eq:tndef}
T_n=T_n(\gamma,\beta)=
\begin{cases}
2^{n(q/2+\beta-\gamma)}, &\mbox{ if $\gamma <q/2+\beta$},\\
n, &\mbox{ if $\gamma =q/2+\beta$},\\
1, &\mbox{ if $\gamma>q/2+\beta$},
\end{cases}
\ee
and
\be\label{eq:lipconst}
\mathbf{G}_n=\sup_{\ell\in S_n^*, y\in\YY, k=1,\cdots, J}\|\mathcal{U}_k(G)(\ell; \circ, y)\|_\infty,  \qquad \mathbf{L}_n=\sup_{\ell\in S_n^*, y\in\YY, k=1,\cdots, J}\|\mathcal{U}_k(G)(\ell; \circ, y)\|_{\mathsf{Lip}(\alpha)}.
\ee
\yadi{$\mathbf{G}_n, \mathbf{L}_n$}{cf. \eqref{eq:lipcond}}
Let
\be\label{eq:gamma_ak_cond}
\gamma\ge \max\left(0, a^*\right),
\ee
$1\le p\le \infty$, and $f\in W_{\max(p,2);\gamma}$.
With  integer $M$ satisfying
\be\label{eq:Mchoice}
M\gs 2^{2n(\gamma-a_*)}(\mathbf{G}_nT_n)^2\left\{c_2+n(\gamma-a_*)+\log(\mathbf{L}_nT_n)\right\}, 
\ee
there exist $M$ tuples $\{(\ell_j, y_j)\}_{j=1}^M$ in $\ZZ_+\times \YY$, and $M$ real numbers $w_j\in\{-1,1\}$ such that for each $k=1,\cdots, J$,
 \be\label{eq:approxest}
 \left\|\mathcal{U}_k(f)-\frac{T_n}{M}\sum_{j=1}^M w_j\mathcal{U}_k(G)(\ell_j;\cdot, y_j)\right\|_p\ls 2^{-n(\gamma-a_k)}
\|f\|_{\max(p,2);\gamma}.
 \ee
  \end{theorem}
 In particular, we obtain the following theorem as a corollary.
 
 \begin{theorem}\label{theo:maintheo}.
 We assume the set up as in Theorem~\ref{theo:maintheo_n}. 
We assume further that there exist $A, B\in\RR$ such that for $n\ge 1$,
\be\label{eq:specialcond}
\mathbf{G}_n \ls 2^{nA}, \qquad \mathbf{L}_n\ls 2^{nB}.
\ee 
 Then for integer $M\ge 2$, there exist $M$ tuples $\{(\ell_j, y_j)\}_{j=1}^M$ in $\ZZ_+\times \YY$, and $M$ real numbers $w_j\in\{-1,1\}$ such that for each $k=1,\cdots, J$,
  \be\label{eq:Mapproxest}
  \left\|\mathcal{U}_k(f)-\frac{T_n}{M}\sum_{j=1}^M w_j\mathcal{U}_k(G)(\ell_j;\cdot, y_j)\right\|_p\ls 
\|f\|_{\max(p,2);\gamma}  \begin{cases}
\disp\left(\frac{\log M}{M}\right)^{(\gamma-a_k)/(q+2\beta+2A-2a_*)}, &\mbox{ if $\gamma<q/2+\beta$},\\
\disp\left(\frac{(\log M)^3}{M}\right)^{(\gamma-a_k)/(2A+2\gamma-2a_*)}, &\mbox{ if $\gamma=q/2+\beta$},\\
\disp\left(\frac{\log M}{M}\right)^{(\gamma-a_k)/(2A+2\gamma-2a_*)}, &\mbox{ if $\gamma>q/2+\beta$},
  \end{cases}
 \ee
 where $n$ is chosen to be the largest integer satisfying \eqref{eq:Mchoice}.
\end{theorem}

\begin{remark}\label{rem:pure_approx}
{\rm
If one is interested only in the approximation of $f$ alone, rather than a simultaneous approximation of $f$ and its  ``derivatives'', the estimates in the above theorem should be used with $a_*=a^*=0$.
\qed}
\end{remark}
\begin{remark}\label{rem:different_ders}
{\rm
In contrast to classical approximation by polynomials, the ``derivatives'' $\mathcal{U}_k(G)$ might not be asymmetric eignets in the sense of our Definition~\ref{def:asymetric_eignet}. 
So, a bound on the approximation of one of these might not be carried over by induction to other ``derivatives''.
In particular, some applications require Birkhoff interpolation/approximation where one might be interested in the approximation of a number of partial differential operators applied to the eignet (e.g. \cite{chandra2015minimum}). 
The statement of the theorems above allows us choose $M$ (and the parameters $w_j$, $\ell_j$, $y_j$) which will work for all these operators simultaneously, although, of course, the estimates for the individual operators will depend upon each operator.
\qed}
\end{remark}

We note some corollaries of Theorem~\ref{theo:maintheo} applied to the various examples discussed in Section~\ref{bhag:asymeignets}. 

The general SVD kernels in Example~\ref{uda:randomfeature} satisfy the conditions of Theorem~\ref{theo:maintheo} with $A=0$. 
Thus, we obtain the following theorem.

\begin{theorem}
\label{theo:gensvd}
Let $G:\XX\times\YY$ be defined as in \eqref{eq:klexpansion}, with $\Lambda_\ell\ls \ell^\beta$ for some $\beta>0$. 
Let $\{\mathcal{U}_1,\cdots, \mathcal{U}_J\}$ be a set of  operators, each compatible with $G$ and derivative-like, with $a_k$ being the exponent for $\mathcal{U}_k$, $k=1,\cdots, J$.
We assume that for some $\alpha>0$,
\be\label{eq:svdlipcond}
|\mathcal{U}_k(G)(x,y)-\mathcal{U}_k(G)(x',y)|\le L_G\rho(x,x')^\alpha, \qquad x, x'\in\XX,\ y\in\YY, \ k=1,\cdots, J.
\ee 
Let $1\le p\le\infty$, $\gamma$ satisfy \eqref{eq:gamma_ak_cond}, and $f\in W_{\max(p,2);\gamma}$.

For integer $M\ge 2$, there exist $M$ tuples $\{ y_j\}_{j=1}^M\subseteq \YY$, and $M$ real numbers $w_j$ such that for each $k=1,\cdots, J$,
\be\label{eq:svdapprox}
\begin{aligned}
\Big\|\mathcal{U}_k(f)&-\sum_{j=1}^M w_j\mathcal{U}_k(G)(\cdot, y_j)\Big\|_p\\
&\ls  
\|f\|_{\max(p,2);\gamma}  \begin{cases}
\disp\left(\frac{\log M}{M}\right)^{(\gamma-a_k)/(q+2\beta-2a_*)}, &\mbox{ if $\gamma<q/2+\beta$},\\
\disp\left(\frac{(\log M)^3}{M}\right)^{(\gamma-a_k)/(2\gamma-2a_*)}, &\mbox{ if $\gamma=q/2+\beta$},\\
\disp\left(\frac{\log M}{M}\right)^{(\gamma-a_k)/(2\gamma-2a_*)}, &\mbox{ if $\gamma>q/2+\beta$}.
  \end{cases}\end{aligned}
\ee
\end{theorem}

Our next theorem deals with twisted zonal function networks (Example~\ref{uda:sbf}).
Here, we may apply Theorem~\ref{theo:maintheo} with $A=0$ to obtain the following theorem.

\begin{theorem}\label{theo:zfapprox}
Let $\beta>0$, and $G :[-1,1]\to \RR$ be a zonal function such that \eqref{eq:zonaltype} is satisfied.
Let $\{\mathcal{U}_1,\cdots, \mathcal{U}_J\}$ be a set of  operators, each compatible with $G$ and derivative-like, with $a_k$ being the exponent for $\mathcal{U}_k$, $k=1,\cdots, J$.
We assume that for some $\alpha>0$,
\be\label{eq:Glipcond}
|\mathcal{U}_k(G)(\x\cdot\y)-\mathcal{U}_k(G)(\x'\cdot\y)|\le L_G\rho(\x,\x')^\alpha, \qquad \x,\x',\y\in\SS^q, \ k=1,\cdots, J.
\ee 
  Let  $R_{j'}\in SO(q+1)$, $j'=1,\cdots,m$.
Let $\gamma>0$, $1\le p\le\infty$,  and $f\in W_{\max(p,2);\gamma}(\SS^q)$, where $\gamma$ satisfies \eqref{eq:gamma_ak_cond}.

For integer $M\ge 1$, there exist $w_{j,j'}\in\RR$, $\y_{j,j'}\in\SS^q$, $j=1,\cdots, M$,  $j'=1,\cdots,m$ such that
\be\label{eq:zfapprox}
\begin{aligned}
\Big\|\mathcal{U}_k(f)(\x)&-\frac{1}{m}\sum_{j'=1}^m \sum_{j=1}^M w_{j,j'}\mathcal{U}_k(G)(\x\cdot R_{j'}\y_{j,j'})\Big\|_p\\
 &\ls 
\|f\|_{\max(p,2);\gamma}  \begin{cases}
\disp\left(\frac{\log M}{M}\right)^{(\gamma-a_k)/(q+2\beta-2a_*)}, &\mbox{ if $\gamma<q/2+\beta$},\\
\disp\left(\frac{(\log M)^3}{M}\right)^{(\gamma-a_k)/(2\gamma-2a_*)}, &\mbox{ if $\gamma=q/2+\beta$},\\
\disp\left(\frac{\log M}{M}\right)^{(\gamma-a_k)/(2\gamma-2a_*)}, &\mbox{ if $\gamma>q/2+\beta$}.
  \end{cases}
  \end{aligned}
\ee
\end{theorem} 

\begin{remark}\label{rem:rotationrem}
{\rm
A slight modification of our proof would allow us to choose a subset of the rotations as well, but this does not add anything new conceptually. 
So, we will actually prove Theorem~\ref{theo:zfapprox} with $m=1$.
The more general version presented here is then obvious, and allows us to use directly the recipe theorem, Theorem~\ref{theo:maintheo}.
\qed}
\end{remark}

Our next theorem applies Theorem~\ref{theo:maintheo} to generalized translation networks (Example~\ref{uda:gentransnet}). 
With the matrices as in this example, it is not difficult to verify that $A=a^*$ and $B=a^*+1$.

\begin{theorem}\label{theo:gtntheo}
Let $q\ge d\ge 1$ be integers, $\alpha\in (0,1]$, and $G\in C(\TT^d)$ satisfy
\be\label{eq:fourcoeffcond}
\hat{G}(\bs 1)=\int_{\TT^d}G(\y)\exp\left(-i\bs 1\cdot \y\right)d\mu_d^*(\y)\not=0.
\ee
Let $\{\mathcal{U}_1,\cdots, \mathcal{U}_J\}$ be a set of  operators, each compatible with $G$ and derivative-like, with $a_k$ being the exponent for $\mathcal{U}_k$, $k=1,\cdots, J$.
We assume Lipschitz-H\"older condition
\be\label{eq:gtnlip}
|\mathcal{U}_k(G)(\x)-\mathcal{U}_k(G)(\x')|\le L_G\rho(\x,\x')^\alpha, \qquad \x,\x'\in\TT^q.
\ee
Let $\gamma>0$ satisfy \eqref{eq:gamma_ak_cond}, and $f\in W_{\max(p,2);\gamma}(\TT^q)$.
For $M\ge 2$, there exist  $d\times q$ matrices $A_{\bs{\ell}_j}$ and $\y_j\in \TT^d$, and $w_j\in\CC$ such that
\be\label{eq:gtnapprox}
\max_{\x\in\TT^q}\left|\mathcal{U}_k(f)(\x)-\sum_{j=1}^M w_j\mathcal{U}_k(G)(A_{\bs{\ell}_j}\x-\y_j)\right|\ls 
\|f\|_{\max(p,2);\gamma}  \begin{cases}
\disp\left(\frac{\log M}{M}\right)^{(\gamma-a_k)/(q+2a^*-2a_*)}, &\mbox{ if $\gamma<q/2$},\\
\disp\left(\frac{(\log M)^3}{M}\right)^{(\gamma-a_k)/(2a^*+2\gamma-2a_*)}, &\mbox{ if $\gamma=q/2$},\\
\disp\left(\frac{\log M}{M}\right)^{(\gamma-a_k)/(2a^*+2\gamma-2a_*)}, &\mbox{ if $\gamma>q/2$}.
  \end{cases}
\ee
\end{theorem}

\begin{remark}\label{rem:fullrank}
{\rm
 In \cite{mhaskar1995degree}, we have required the matrices $A_{\bs\ell}$ to be full rank. 
 One consequence of our theorem is to relax this condition to \eqref{eq:fourcoeffcond}. Applied to the case of neural networks ($d=1$), this condition is exactly the one which is necessary and sufficient for neural networks to be universal approximators. 
 So, our theorem generalizes our old result on neural networks.
 \qed
 }
 \end{remark}

\bhag{ReLU$^r$ networks}\label{bhag:relu}

In this section, we examine the problem of approximation by shallow ReLU$^r$ networks for functions in $W_{\gamma,p}$.

 In the proof of Theorems~\ref{theo:maintheo_n} and \ref{theo:maintheo}, we have used H\"offding's inequality in a straightforward manner.
 Of course, more sophisticated ways of applying concentration inequalities are available in the literature under various conditions on $G$ and the target function.
 The argument leading to an estimation of $|\nu|_{TV}$ as in Lemma~\ref{lemma:tvlemma} can be used to get different estimates in these cases.
 In light of the recent interest in activation functions of the form $t\mapsto t_+^r$ in connection with neural networks, we illustrate the use of these ideas in the case of these activation functions.
  Unlike most of the other papers in this direction, we let $r>0$ to be any positive number, not just an integer.
The mathematics involved is more pleasant if we consider the activation function \yadi{$G_r$}{Activation function for ReLU$^r$ networks \eqref{eq:reluact}}
\be\label{eq:reluact}
G_r(t)=\begin{cases}
|t|^r, &\mbox{if $r$ is not an  integer},\\
(\max(t, 0))^r, &\mbox{if $r$ is an  integer}.
\end{cases}
\ee
Accordingly, we consider in this section the case when $(\x,\y)\mapsto G_r(\x\cdot\y)$ is the symmetric kernel  defined on $\XX=\YY=\SS^q$. 
We recall the fact that the measure $\mu^*_\XX=\mu_q^*$ is the Riemannian volume measure on $\SS^q$, normalized to be a probability measure.
 
It is argued in \cite{bach2014, dingxuanpap} that the problem of approximation by neural networks with such homogeneous activation functions is considered fruitfully as the problem of approximation by zonal function networks. 
Thus, we map $\RR^q$ to the unit sphere:
$$
\SS^q=\{\mathbf{x}\in\RR^{q+1} : |\mathbf{x}|_{q+1}=1\}, 
$$
and its upper hemisphere:
$$
\SS^q_+=\{\x\in \SS^q: x_{q+1}>0\},
$$
using the mapping $\pi^* : \RR^q\to\SS^q_+$
 given by \yadi{$\pi^*$}{coordinate map for upper hemisphere $\SS^q_+$, \eqref{eq:euclid_to_sphere}}
\be\label{eq:euclid_to_sphere}
\pi^*(x_1,\cdots,x_q)= \left(\frac{x_1}{\sqrt{1+|\x|_2^2}},\cdots,\frac{x_q}{\sqrt{1+|\x|_2^2}}, \frac{1}{\sqrt{1+|\x|_2^2}}\right).
\ee
We note that 
\be\label{eq:sphere_to_euclid}
(\pi^*)^{-1}(u_1,\cdots,u_{q+1})=\left(\frac{u_1}{u_{q+1}}, \cdots, \frac{u_q}{u_{q+1}}\right), \qquad \u\in\SS^q.
\ee
A neural network of the form $\x\mapsto \sum_{j=1}^M a_jG_r(\x\cdot\w_k+b_k)$, $\x\in\RR^q$, is mapped to
$\u\mapsto u_{d+1}^{-r/2}\sum_{j=1}^M a_j'G_r(\u\cdot \v_j)$, $u\in \SS^q_+$ with $\v_j$ being the unit vector along $(\w_j,b_j)$. 
Because of our definition of $G_r$, this network can be extended to $\SS^q$ as an even or odd function as appropriate. 
Likewise, for any $F:\RR^q\to\RR$ such that $F(\x)(1+|\x|_2^2)^{r/2}\in C_0(\RR^q)$\footnote{The space $C_0(\RR^q)$ is the space of all functions continuous on $\RR^q$ which vanish at infinity, the space being equipped with the uniform norm.} corresponds the function $f(\u)=u_{q+1}^{-r/2}F((\pi^*)^{-1}\u)$ defined on $\SS^q$, where $\u=\pi^*(\x)$. 
Again, we may extend $f$ to $\SS^q$ as an even or odd function as needed.
Thus, the problem of approximation of $F$ in a weighted $L^p$ norm on $\RR^q$ is equivalent to the approximation of $f$ by networks of the form $\sum_{j=1}^M a_j'G_r(\u\cdot \v_j)$ on $\SS^q$.

We will state our theorem first for integer values of $r$.
\begin{theorem}
\label{theo:relutheorem}
Let $q\ge 1$ be an integer, $r\ge 1$ be an integer, $\gamma>0$,  $\{\mathcal{U}_1,\cdots, \mathcal{U}_J\}$ be a set of  pseudo-differential or differential operators,  with integer
$a_k<r$ being the exponent for $\mathcal{U}_k$, $k=1,\cdots, J$.
Let $1\le p\le \infty$,  $\gamma >a^*$, $f\in W_{\max(p,2);\gamma}(\SS^q)$.
Then for integer $M\ge 2$, there exist $\y_j\in\SS^q$, $w_j\in\RR$, $j=1,\cdots, M$ such that for each $k=1,\cdots, J$, we have
\be\label{eq:reluest}
\begin{aligned}
\Biggl\|\mathcal{U}_k(f)&-\mathcal{U}_k\big(\sum_{j=1}^M w_jG_r(\circ\cdot\y_j)\big)\Biggr\|_p\\
&\ls \|f\|_{W_{\max(p,2);\gamma}(\SS^q)}\begin{cases}
\disp\frac{\sqrt{\log M}}{M^{(\gamma-a_k)/q}} &\mbox{ if $\gamma< (q+2r+1)/2$}\\[1ex]
\disp\frac{(\log M)^{3/2}}{M^{(\gamma-a_k)/q}}, &\mbox{ if $\gamma= (q+2r+1)/2$}, \\[1ex]
\disp\frac{\sqrt{\log M}}{M^{(q+2r+1-2a_k)/(2q)}}, &\mbox{ if $\gamma> (q+2r+1)/2$}.
\end{cases}
\end{aligned}
\ee
\end{theorem}

\begin{remark}\label{rem:maotong}
{\rm
For the ReLU networks, $r=1$, and the error terms in estimate \eqref{eq:reluest}, used with $\mathcal{U}_k=id$, becomes
$$
\begin{cases}
\disp\frac{\sqrt{\log M}}{M^{\gamma/q}} &\mbox{ if $\gamma< (q+3)/2$}\\[1ex]
\disp\frac{(\log M)^{3/2}}{M^{\gamma/q}}, &\mbox{ if $\gamma= (q+3)/2$}, \\[1ex]
\disp\frac{\sqrt{\log M}}{M^{(q+3)/(2q)}}, &\mbox{ if $\gamma> (q+3)/2$}.
\end{cases}
$$
In the case when $\gamma\le (q+3)/2$, these are sharper than the estimates \eqref{eq:intro_maobd}. 
In the case when $\gamma>(q+3)/2$, these bounds coincide with the bounds we obtained in \cite{mhaskar2020dimension}. 
In particular, they are sharper than the estimates \eqref{eq:intro_maobd} in the overlapping regime $\gamma< (q+4)/2$.
Moreover, our bounds are valid for approximation in any $L^p$ space, $1\le p\le \infty$.
\qed}
\end{remark}
For non-integer values of $r$, the estimates are similar, but slightly worse.
 
\begin{theorem}
\label{theo:nonrelutheo}
We assume the set up as in Theorem~\ref{theo:relutheorem}, except for the assumption that $r$ is not an integer.
Then the conclusion \eqref{eq:reluest} holds with the following modification.
For any $\delta>0$, we have
\be\label{eq:nonreluest}
\begin{aligned}
\Biggl\|\mathcal{U}_k(f)&-\mathcal{U}_k\big(\sum_{j=1}^M w_jG_r(\circ\cdot\y_j)\big)\Biggr\|_p\\
&\ls \|f\|_{W_{\max(p,2);\gamma}(\SS^q)}M^\delta\begin{cases}
\disp\frac{\sqrt{\log M}}{M^{(\gamma-a_k)/q}} &\mbox{ if $\gamma< (q+2r+1)/2$}\\[1ex]
\disp\frac{(\log M)^{3/2}}{M^{(\gamma-a_k)/q}}, &\mbox{ if $\gamma= (q+2r+1)/2$}, \\[1ex]
\disp\frac{\sqrt{\log M}}{M^{(q+2r+1-2a_k)/(2q)}}, &\mbox{ if $\gamma> (q+2r+1)/2$},
\end{cases}
\end{aligned}
\ee
where the constants involved may depend upon $\delta$.
\end{theorem}

\bhag{Proofs}\label{bhag:proofs}

As mentioned in the introduction, the idea behind the proofs is to first approximate $f$ by $\sigma_{2^n}(f)\in\Pi_{2^n}$. 
Then we use \eqref{eq:ass_variation} to express $\sigma_{2^n}(f)$ in an integral form, and estimate the total variation of $\nu_{G}(\sigma_{2^n}(f)))$.
A careful application of the H\"offding concentration inequality (or in the case of ReLU$^r$ networks, a different variant of the same) leads to the proof.

We begin this  program by recalling H\"offding inequality in Lemma~\ref{lemma:hoffding}, and using it for general estimates on the supremum norm  in Lemma~\ref{lemma:hoffdingspecial}. 
In Lemma~\ref{lemma:tvlemma}, we estimate the TV norms of certain measures as required in Lemma~\ref{lemma:hoffdingspecial}.  
The proof is then completed by putting everything together.

\begin{lemma}\label{lemma:hoffding}{\rm\textbf{H\"offding's inequality} (\cite[Theorem~2.8]{boucheron2013concentration})}
Let $X_j$, $j=1,\cdots, M$ be independent random variables, with each $X_j\in [a_j,b_j]$. Then 
\be\label{eq:hoffding}
\mathsf{Prob}\left(\left|\frac{1}{M}\sum_{j=1}^M (X_j-\mathbb{E}(X_j))\right|>t\right)\le 2\exp\left(-2\frac{M^2t^2}{\disp\sum_{j=1}^M (b_j-a_j)^2}\right).
\ee
\end{lemma}

\begin{remark}\label{rem:hoeffding}
{\rm
We wish to apply Lemma~\ref{lemma:hoffding} to quantities of the form $G(\circ,x,\circ)$, treating these as random variables defined on $\ZZ_+\times \YY$. 
This would  yield estimates for each $x$.
The following lemma shows how to extend these estimates to the uniform norm on $\XX$.
In order not to introduce pedantic notation, we make the following abuse of terminology.
We assume a measure space $\Omega$ with a probability measure $\nu^*$.
An $\Omega$-valued random variable (transformation) is a measurable function $Z:\mathcal{Z}\to\Omega$ where  $(\mathcal{Z}, \tilde{\nu})$ is another probability space, in the sense that $Z^{-1}(B)$ is $\tilde{\nu}$-measurable for every $\nu^*$ measurable $B$ \cite[Chapter~VIII]{halmos2013measure}. 
The random variable is distributed according to the law $\nu^*$ if $\tilde{\nu}(Z^{-1}(B))=\nu^*(B)$ for all $\nu^*$-measurable $B$.
Two such random variables $Z_1, Z_2$ are independent if 
$\tilde{\nu}(Z_1^{-1}(B_1)\cap Z_2^{-1}(B_2))=\nu^*(B_1)\nu^*(B_2)$.
If $Z_1,\cdots, Z_m$ are independent $\Omega$-valued random variables, then for each $z\in \mathcal{Z}$, $(Z_1(z),\cdots, Z_M(z))=(z_1,\cdots, z_m)\in \Omega^m$ is a random sample (or random variable depending upon our point of view). 
In the following lemma, we will use the notation $F_k(\circ, Z_j)$ for a function $F_k :\XX\times\Omega\to\RR$ as a shorthand notation for the random variable  $z\mapsto F_k(\circ, Z_j(z))$, rather than defining a new function on the product of $\XX$ a space of $\Omega$-valued functions on $\mathcal{Z}$.
\qed}
\end{remark}

\begin{lemma}\label{lemma:hoffdingspecial}
Let $\Omega$ be a measure space, $\nu$ be a measure on $\Omega$ with $0<|\nu|_{TV}<\infty$, and $\{Z_1,\cdots,Z_M\}$ are $\Omega$-valued random variables,  drawn from the probability law $|\nu|/|\nu|_{TV}$. 
Let $\alpha>0$. 
For each $\omega\in\Omega$, let $\{F_k(\cdot,\omega):\XX\to\RR\}_{k=1}^J$ be a family of functions in $\mathsf{Lip}(\alpha)$, such that
\be\label{eq:Flipconst}
\mathbf{R}=\sup_{\omega\in\Omega, k=1,\cdots,J}\|F_k(\cdot,\omega)\|_\infty<\infty, \qquad\mathbf{F}=\sup_{\omega\in\Omega, k=1,\cdots,J}\|F_k(\cdot,\omega)\|_{\mathsf{Lip}(\alpha)} <\infty.
\ee
Then there are numbers $w_j$, $j=1,\cdots,M$, such that $|w_j|=1$,  for every $k=1,\cdots, J$ and any $t>0$,
\be\label{eq:supnormprob}
\mathsf{Prob}\left(\left\|\frac{|\nu|_{TV}}{M}\sum_{j=1}^M w_j F_k(\cdot,Z_j) -\int_\Omega F_k(\cdot,\omega)d\nu(\omega)\right\|_\infty > t\right)\ls (\mathbf{F}|\nu|_{TV})^{q/\alpha}t^{-q/\alpha}\exp\left(-c\frac{Mt^2}{\mathbf{R}^2|\nu|_{TV}^2}\right),
\ee
where the probability is defined  on a product latent space in terms of $|\nu|/|\nu|_{TV}$ as explained in Remark~\ref{rem:hoeffding}.
\end{lemma}

\begin{proof}\ 
Let $1>\epsilon>0$ to be chosen later. 
Since $\XX$ is compact, we may find a maximal $\epsilon$-separated subset $\C$ of $\XX$; i.e., a maximal set $\C$ of points such that if $x,y\in\C$ and $x\not=y$, then $\rho(x,y)\ge \epsilon$. 
The maximality of $\C$ ensures that $\disp\XX=\bigcup_{x\in\C}\mathbb{B}(x,\epsilon)$, and that the balls   $\mathbb{B}(x,\epsilon/3)$, $x\in\C$, are all disjoint.
Since $\mu^*(\XX)=1$ and $\mu^*(\mathbb{B}(x,\epsilon))\sim \epsilon^q$ (cf. \eqref{eq:measureequivalence}) for each $x$, it follows that 
\be\label{eq:pf2eqn1}
|\C|\sim \epsilon^{-q}.
\ee 

In this proof, the quantities $w_j$ and $x^*$ to be introduced below depend upon the specific realization of the random variables $Z_j$, and as such, are random variables themselves. 
Rather than complicating the notation with capital and small case letters, we will use small case letters $z_j$ in place of $Z_j$ with this understanding; take the viewpoint that $\{z_j\}\subset \Omega$ is a random sample, and $w_j\in\{-1,1\}$, $x^*\in \Omega$ depend upon this sample.

We note that $\nu$ is a signed measure with $|\nu|_{TV}=|\nu|(\XX)<\infty$. 
In this proof only, let $g :\YY\to\RR$ be a function with $|g(y)|=1$ for all $y\in\YY$, such that $d\nu(t)=g(t)d|\nu|(t)$; i.e., $g$ be the Radon-Nikodym derivative of $\nu$ with respect to $|\nu|$. The Radon-Nikodym derivative exists only $\nu$-almost everywhere, but our formulation means that $g$ is extended to the $\nu$-null set so as to satisfy the constraint everywhere. 
We write $w_j=g(z_j)$. 

In this proof let $F$ be any of the functions $F_k$, and $\nu^*=|\nu|/|\nu|_{TV}$ be the probability measure associated with $\nu$.
Recalling that  each $|F(\cdot,\omega)|\le \mathbf{R}$, we deduce from H\"offding's inequality  (with $g(z_j)F(x',z_j)$ in place of $X_j$) that for each $x'\in\C$,
$$
\mathsf{Prob}\left(\left|\frac{1}{M}\sum_{j=1}^M w_j  F(x',z_j) -\int_\Omega g(\omega)F(x',\omega)d\nu^*(\omega)\right| > t/(3|\nu|_{TV})\right)\ls \exp\left(-\frac{2Mt^2}{9\mathbf{R}^2|\nu|_{TV}^2}\right).
$$
Since $gd\nu^*=d\nu/|\nu|_{TV}$, we can rewrite this estimate as
\be\label{eq:pf2eqn2}
\mathsf{Prob}\left(\left|\frac{|\nu|_{TV}}{M}\sum_{j=1}^M w_j F(x',z_j) -\int_\Omega F(x',\omega)d\nu(\omega)\right| > t/3\right)\ls \exp\left(-\frac{2Mt^2}{9\mathbf{R}^2|\nu|_{TV}^2}\right).
\ee
In view of \eqref{eq:pf2eqn1} and the so called ``union bound'',  this leads to 
\be\label{eq:pf2eqn3}
\mathsf{Prob}\left(\max_{x'\in\C}\left|\frac{|\nu|_{TV}}{M}\sum_{j=1}^M w_j F(x',z_j) -\int_\Omega F(x',\omega)d\nu(\omega)\right| > t/3\right)\ls \epsilon^{-q}\exp\left(-\frac{2Mt^2}{9\mathbf{R}^2|\nu|_{TV}^2}\right).
\ee

Since the function
$$
x\mapsto \frac{|\nu|_{TV}}{M}\sum_{j=1}^M w_j F(x,z_j) -\int_\Omega F(x,\omega)d\nu(\omega)
$$
is continuous on $\XX$, it attains its maximum at some $x^*$. 
Our conditions on the Lipschitz norm of $F(\cdot,\omega)$ now implies that there exists $x'\in\C$ such that
\be\label{eq:pf2eqn4}
|\nu|_{TV}\sup_{\omega\in\Omega}|F(x^*,\omega)-F(x',\omega)|\ls \mathbf{F}|\nu|_{TV}\epsilon^\alpha. 
\ee
We choose $\epsilon=c(t/\mathbf{F}|\nu|_{TV})^{1/\alpha}$ for a suitable constant $c$, such that the right hand side of \eqref{eq:pf2eqn4} is $=t/3$.
Then
$$
\left|\frac{|\nu|_{TV}}{M}\sum_{j=1}^M w_j F(x^*,z_j) -\int_\Omega F(x^*,\omega)d\nu(\omega)\right|\le \frac{2t}{3}+ \max_{x'\in\C}\left|\frac{|\nu|_{TV}}{M}\sum_{j=1}^M w_j F(x',z_j) -\int_\Omega F(x',\omega)d\nu(\omega)\right|.
$$
Therefore, \eqref{eq:pf2eqn3} leads to
$$
\mathsf{Prob}\left(\max_{x\in\XX}\left|\frac{|\nu|_{TV}}{M}\sum_{j=1}^M w_j F(x,z_j) -\int_\Omega F(x,\omega)d\nu(\omega)\right| > t\right)\ls (\mathbf{F}|\nu|_{TV})^{q/\alpha}t^{-q/\alpha}\exp\left(-c\frac{Mt^2}{\mathbf{R}^2|\nu|_{TV}^2}\right).
$$
It is easy to deduce \eqref{eq:supnormprob} by  applying this estimate to each $F_k$. 
\end{proof}

\begin{lemma}\label{lemma:tvlemma}
Let $n\ge 0$ be integer, $G$ be an asymmetric eignet kernel with exponents $(\alpha,\beta)$. Let $P\in\Pi_{2^n}$. Then (cf. \eqref{eq:ass_measure})
\be\label{eq:polytvestimate}
|\nu_G(P)|_{TV}\ls \sum_{j=0}^{n+1} 2^{j(q/2+\beta)}\|\tau_j(P)\|_2.
\ee
\end{lemma}

\begin{proof}\ 
Since $h(t)=1$ if $t\le 1/2$ and $=0$ if $t\ge 1$, we have
 $$
 \begin{aligned}
 P&=\sum_{k:\lambda_k<2^n} \hat{P}(k)\phi_k =\sum_{k:\lambda_k<2^n}H\left(\frac{\lambda_k}{2^{n+1}}\right) \hat{P}(k)\phi_k\\
  &= \sum_{k:\lambda_k<2^n}\left\{H(\lambda_k)+\sum_{j=1}^{n+1}\left(H\left(\frac{\lambda_k}{2^j}\right)-H\left(\frac{\lambda_k}{2^{j-1}}\right)\right)\right\}  \hat{P}(k)\phi_k=\sum_{j=0}^{n+1} \tau_j(P).
  \end{aligned}
 $$
Hence,
\be\label{eq: pf1eqn1}
\nu_G(P)(\ell,y)=\sum_{j=0}^{n+1} \nu_G(\tau_j(P))(\ell,y),
\ee
We note that for each $j$, (cf. \eqref{eq:indexset})
\be\label{eq:pf1eqn2}
\nu_G(\tau_j(P))(\ell,y)=0, \qquad \ell\not\in \mathbf{S}_j.
\ee
In this proof, we introduce the notation 
$$
v_j(\ell,y)=\mathcal{D}_G\phi_\ell(y)\widehat{\tau_j(P)}(\ell),
$$
so that $\nu_G(\tau_j(P))(\ell,y)=v_j(\ell,y)d\mu^*_\YY(y) d\mathfrak{c}(\ell)$.
Using Schwarz inequality and the fact that $\mu_\YY^*$ is a probability measure, we obtain that
\be\label{eq:pf1eqn3}
\begin{aligned}
|\nu_G(\tau_j(P))|_{TV}=\int_\YY\sum_{\ell\in \mathbf{S}_j}&|v_j(\ell,y)|d\mu_\YY^*(y)\le \int_\YY \left\{\sum_{\ell\in \mathbf{S}_j}|\widehat{\tau_j(P)}(\ell)|^2\right\}^{1/2}\left\{\sum_{\ell\in \mathbf{S}_j}|\mathcal{D}_G\phi_\ell(y)|^2\right\}^{1/2}d\mu_\YY^*(y)\\
&\le \|\tau_j(P)\|_2\left\{\int_\YY \sum_{\ell\in \mathbf{S}_j}|\mathcal{D}_G\phi_\ell(y)|^2 d\mu_\YY^*(y)\right\}^{1/2}
\end{aligned}
\ee

Since $|\mathbf{S}_j|\ls 2^{jq}$, \eqref{eq:invchristbd} leads to 
$$
\int_\YY \sum_{\ell\in \mathbf{S}_j}|\mathcal{D}_G\phi_\ell(y)|^2 d\mu_\YY^*(y)\ls 2^{jq+2j\beta}.
$$
Hence, we deduce using \eqref{eq:pf1eqn3}  that
\be\label{eq:pf1eqn4}
|\nu_G(\tau_j(P))|_{TV}=\int_\YY\sum_{\ell\in \mathbf{S}_j}|v_j(\ell,y)|d\mu_\YY^*(y)\le 2^{j(q/2+\beta)} \|\tau_j(P)\|_2.
\ee
Together with \eqref{eq: pf1eqn1}, this leads to \eqref{eq:polytvestimate}.
\end{proof}

\begin{cor}\label{cor:tvcor}
Let $\gamma>0$, $f\in W_{2;\gamma}$. We have (cf. \eqref{eq:tndef})\\
\be\label{eq:tvestimate}
|\nu_G(\sigma_{2^n}(f))|_{TV}\ls \|f\|_{W_{2;\gamma}}T_n = \|f\|_{W_{2;\gamma}}\begin{cases}
2^{n(q/2+\beta-\gamma)}, &\mbox{ if $\gamma <q/2+\beta$},\\
n, &\mbox{ if $\gamma =q/2+\beta$}\\
1, &\mbox{ if $\gamma>q/2+\beta$}.
\end{cases}
\ee
\end{cor}
\begin{proof}\ 
We note that $\tau_j(\sigma_{2^n}(f))=\sigma_{2^n}(\tau_j(f))$, so that, in view of Theorem~\ref{theo:equivtheo}, 
$$
\|\tau_j(\sigma_{2^n}(f))\|_2 \ls \|\tau_j(f)\|_2\ls 2^{-j\gamma}\|f\|_{W_{2;\gamma}}.
$$
The corollary is now an easy consequence of Lemma~\ref{lemma:tvlemma}.
\end{proof}\\

\noindent\textsc{Proof of Theorem~\ref{theo:maintheo_n}}

We observe first (cf. Theorem~\ref{theo:goodapprox}, Proposition~\ref{prop:derivative}) that
\be\label{eq:pf3eqn3}
\|\mathcal{U}_k(f)-\mathcal{U}_k(\sigma_{2^n}(f))\|_p\ls  2^{-n(\gamma-a_k)}\|f\|_{W_{p;\gamma}}.
\ee
So, letting $P=\sigma_{2^n}(f)$, it is enough to approximate (cf. \eqref{eq:ass_variation})
$$
\mathcal{U}_k(P)(x)=\mathcal{U}_k(\sigma_{2^n}(f))(x)=\int_{\ZZ_+\times \YY}\mathcal{U}_k(G)(\ell;x,y)d\nu_G(P)(\ell,y).
$$
In view of Corollary~\ref{cor:tvcor} and \eqref{eq:tvestimate}, we see that $|\nu_G(P)|_{TV}\ls T_n$.
 We now apply Lemma~\ref{lemma:hoffdingspecial} with $\ZZ_+\times \YY$ in place of $\Omega$, $\mathcal{U}_k(G)$ in place of $F_k$, $\mathbf{G}_n$ in place of $\mathbf{R}$,  $\mathbf{L_n}$ in place of $\mathbf{F}$, and $t=2^{-n(\gamma-a_*)}$ ($\le 2^{-n(\gamma-a_k)}$ for all $k$) to deduce that
\be\label{eq:pf1eqn1}
\begin{aligned}
\mathsf{Prob}&\left\{ \left\|\mathcal{U}_k(P)-\frac{|\nu_G(P)|_{TV}}{M}\sum_{j=1}^M w_j\mathcal{U}_k(G)(\ell_j;\cdot, y_j)\right\|>2^{-n(\gamma-a_*)}\right\}\\
&\ls (\mathbf{L}_n T_n)^{q/\alpha}t^{-q/\alpha}\exp\left(-c\frac{M}{2^{2n(\gamma-a_*)}(\mathbf{G}_nT_n)^2}\right).
\end{aligned}
\ee
The choice of $M$ as in \eqref{eq:Mchoice} ensures that the right hand side of the above inequality is $<1$. 
Thus, there exist $w_j, \ell_j, y_j$ such that \eqref{eq:approxest} holds.
\qed\\[1ex]

\noindent\textsc{Proof of Theorem~\ref{theo:maintheo}.}\\[1ex]
In view of \eqref{eq:specialcond} and \eqref{eq:tndef}, we have
\be\label{eq:pf5eqn1}
2^{2n(\gamma-a_*)}(\mathbf{G}_nT_n)^2\left\{c_2+n(\gamma-a_*)+\log(\mathbf{L}_nT_n)\right\}\ls
\begin{cases}
2^{2n(A+q/2+\beta-a_*)}n, & \mbox{ if $\gamma<q/2+\beta$},\\
2^{2n(\gamma-a_*+A)}n^3, & \mbox{ if $\gamma=q/2+\beta$},\\
2^{2n(\gamma-a_*+A)}n, &\mbox{if $\gamma<q/2+\beta$}.
\end{cases}
\ee
It is easy to verify that for any $a, b>0$, writing 
$$
x=\left(\frac{y}{(\log y)^b}\right)^{1/a}, \qquad y>1,
$$
we have
$$
x^a(\log x)^b \sim y.
$$
Using this fact with $x=2^n$, we see that \eqref{eq:Mchoice} is satisfied if we choose 
$$
2^n \sim \left(\frac{M}{(\log M)^b}\right)^{1/a}
$$
where $b=3$ if $\gamma=q/2+\beta$ and $b=1$ otherwise, and
$$
a=\begin{cases}
2A+q+2\beta-2a_* & \mbox{ if $\gamma<q/2+\beta$},\\
2\gamma-2a_*+2A, & \mbox{ if $\gamma=q/2+\beta$},\\
2\gamma-2a_*+2A, &\mbox{if $\gamma<q/2+\beta$}.
\end{cases}
$$
The estimate \eqref{eq:approxest} with this choice of $n$ leads to \eqref{eq:Mapproxest}.
 \qed\\

In  order to prove Theorem~\ref{theo:relutheorem}, we first recall some results in the following lemma.
Lemma~\ref{lemma:relu}(a) can be deduced easily from the Rodrigues' formula (cf.  \cite[Lemma~3.1]{mhaskar2000polynomial}). 
Lemma~\ref{lemma:relu}(b) is proved in \cite[Proposition~A.1]{sphrelu} (with a different notation where $2\gamma+1$ is used for $r$).

\begin{lemma}
\label{lemma:relu}
{\rm (a)} If $r\ge 1$ is an integer, then we have the formal expansion
\be\label{eq:intreluexp}
(\max(t,0))^r \sim Q_r(t)+\frac{\Gamma(q/2)\Gamma(r+1)}{2^{r+1}\sqrt{\pi}}\sum_{\ell=0}^\infty (-1)^m\frac{\Gamma(\ell+1/2)}{\Gamma(\ell+1/2+(q+2r+1)/2)}p_{2\ell+r+1}^{(q/2-1,q/2-1)}(1)p_{2\ell+r+1}^{(q/2-1,q/2-1)}(t),
\ee
where $Q_r$ is an algebraic polynomial of degree $\le r$.\\
{\rm (b)} If $r$ is not an integer, then we have the formal expansion
\be\label{eq:nonintreluexp}
|t|^r\sim   \frac{\cos(\pi(r-1)/2)\Gamma(q/2)\Gamma(r+1)}{2^r\sqrt{\pi}}\sum_{\ell=0}^\infty (-1)^\ell \frac{\Gamma(\ell+q/2)\Gamma(\ell-r/2-1)}{\Gamma(\ell+1/2)\Gamma(\ell+(q+r+1)/2)}
p_{2\ell }^{(q/2-1,q/2-1)}(1)p_{2\ell }^{(q/2-1,q/2-1)}(t).
\ee
\end{lemma}
\noindent\textsc{Proof of Theorem~\ref{theo:relutheorem}}.\\

We recall the notation from  Example~\ref{uda:specialmanifold_sph}.
In particular, $\Pi_n^q$ is the space of all spherical polynomials of degree $<n$.
The operator $\mathcal{D}_{G_r}$ is defined analogously to \eqref{eq:basicass}, so that \eqref{eq:pf4eqn1} below holds, with  the coefficients  $b_\ell$  defined using the expansions in Lemma~\ref{lemma:relu} and the addition formula \eqref{eq:addformula}.
 In view of Lemma~\ref{lemma:relu}(a), we see that for any polynomial $P\in \Pi_n^q$, 
\be\label{eq:pf4eqn1}
\mathcal{D}_{G_r}(P)(\x)=\mathcal{D}_{G_r}(Q_r)(\x)+\sum_{\ell=0}^{n-1}(-1)^\ell b_\ell^{-1}\sum_k \hat{P}(2\ell+r+1,k)Y_{2\ell+r+1,k},
\ee
where $b_\ell^{-1}\sim \ell^{(q+2r+1)/2}$.
In particular, \eqref{eq:invchristbd} holds with $\beta=(q+2r+1)/2$.
If $2^{j-1}\ge r+1$ then $\tau_j(Q_r)=0$.
If $f\in W_{2;\gamma}$, we deduce  that for $2^{j-1}\ge r+1$,
\be\label{eq:pf4eqn2}
\|\mathcal{D}_{G_r}(\tau_j(f))\|_2^2\ls 2^{j(q+2r+1)}\|\tau_j(f)\|_2^2 \ls 2^{j(q+2r+1)}\|\tau_j(f)\|_2^2.
\ee
This estimate holds trivially for $2^{j-1}<r+1$.
We note that $\mathcal{D}_{G_r}$ commutes with  $\sigma_{2^n}$ and all $\tau_j$'s.
Hence (cf. Theorem~\ref{theo:equivtheo}, \eqref{eq:l2equivrel}),
$$
\|\mathcal{D}_{G_r}(\sigma_{2^n}(f))\|_1^2\le \|\mathcal{D}_{G_r}(\sigma_{2^n}(f))\|_2^2 \ls \sum_{j=0}^n 2^{j(q+2r+1)}\|\tau_j(f)\|^2\ls \|f\|_{W_{2;\gamma}}^2\sum_{j=0}^n 2^{j(q+2r+1-2\gamma)};
$$
i.e.,
\be\label{eq:pf4eqn3}
\|\mathcal{D}_{G_r}(\sigma_{2^n}(f))\|_1 \ls \|f\|_{W_{2;\gamma}}\left\{\sum_{j=0}^n 2^{j(q+2r+1-2\gamma)}\right\}^{1/2}.
\ee

Next, we recall the results from \cite{mhaskar2020dimension}.
If $\mathcal{U}_k$ is a pseudo-differential operator with exponent $a_k$, then $\mathcal{U}_k(G_r)$ has an expansion analogous to \eqref{eq:intreluexp} whose coefficients $b_\ell$ satisfy 
$$
(-1)^\ell b_\ell = \ell^{-(q+2r-2a_k+1)/2}\left(1+\O(1/\ell)\right),
$$ the same as that of $G_{r-a_k}$.
If $\mathcal{U}_k$ is a derivative of order $a_k$, then $\mathcal{U}_k(G_r)$ is a finite linear combination of $G_{r-a_k+j}$, $j=0,1,\cdots, a_k$ with trigonometric polynomials as the coefficients in this combination.
In either case, $\mathcal{U}_k(G_r)$ is H\"older continuous with exponent $r-a_k$, $r-a_k$ smooth on $\SS^q$, and for each $\x\in\SS^q$, infinitely differentiable on $\SS^q\setminus \{\y\in\SS^q: \x\cdot\y=0\}$.
We may use Lemma~\ref{lemma:hoffdingspecial} instead of the usual H\"offding's inequality as in the proof of \cite[Theorem~3.1]{mhaskar2020dimension} to conclude as in 
 \cite[Corollary~4.1]{mhaskar2020dimension} (recalling that $r$ in this paper is $2\gamma+1$ in \cite{mhaskar2020dimension}),  that for any $M\ge 2$, there exists $\y_j\in\SS^q$, $w_j\in\RR$, $j=1,\cdots, M$ such that (cf. \eqref{eq:pf4eqn3})
\be\label{eq:pf4eqn4}
\begin{aligned}
\Biggl\|\mathcal{U}_k(\sigma_{2^n}(f))-\sum_{j=1}^M& w_j\mathcal{U}_k(G_r)(\circ\cdot\y_j)\Biggr\|_p\\
&\ls \|f\|_{W_{\max(p,2);\gamma}}\left\{\sum_{j=0}^n 2^{j(q+2r+1-2\gamma)}\right\}^{1/2}\frac{\sqrt{\log M}}{M^{(q+2r-2a_k+1)/(2q)}}.
\end{aligned}
\ee
Choosing $n$ such that $2^{nq}\sim M$, we obtain the analogue of Corollary~\ref{cor:tvcor}. 
The estimate \eqref{eq:reluest} is proved as before by combining this with Theorem~\ref{theo:goodapprox}.\qed\\[1ex]
\yadi{$\vert\nu\vert$}{total variation measure for $\nu$}
\yadi{$\vert\nu\vert_{TV}$}{Total variation of $\nu$, $|\nu|(\XX)$}
\noindent\textsc{Proof of Theorem~\ref{theo:nonrelutheo}.}\\

The proof of this theorem is verbatim the same as that of Theorem~\ref{theo:relutheorem}, except the results in \cite{mhaskar2020dimension} call  for an extra factor of $M^\delta$ in \eqref{eq:pf4eqn4}.\qed

\bhag{Conclusions}
\label{bhag:conclusions}
In classical theoretical machine learning, it is customary to study the expressive power of kernel based approximations of the form $\sum_k a_k G(x, y_k)$, where $G$ is, for example, the kernel of a reproducing kernel Hilbert space.
Popular neural networks such as ReLU networks can also be formulated in this manner by dimension raising.
A classical tool for this purpose is the Mercer expansion of $G$, and the space of target functions is the so called native (or variation) space for the kernel. 
Purely probabilistic techniques lead typically to dimension independent bounds, while purely approximation theory based techniques lead to constructive methods of approximation, but necessarily suffer from the curse of dimensionality.
Moreover, the native space is often difficult to characterize in terms of interpretable criteria such as the number of derivatives.
Thus, it is an active area of research to investigate the approximation properties of kernel based approximation of functions in Sobolev classes.
A great deal of this research focuses on approximation on known domains such as a cube, sphere, Euclidean space, torus, etc.

 Many emerging applications of machine learning point to the study of approximations of the same form, except that the points $x$ and $y_k$ may belong to different spaces. 
Examples include transfer learning, ISAR imaging, learning with random features, classification based on deformed or transformed data, etc. 
Our contributions in this paper are summarized as:
\begin{itemize}
\item We have studied the expressive power of general kernel based networks  where \emph{the kernels are not symmetric}, and in fact, \emph{may be defined on a product to two different spaces}.
\item The approximation is taken in an abstract setting of data spaces (generalizing data defined manifolds) rather than classical domains such as a cube.
\item Our theorems are very general, applicable to a wide variety of networks, including periodic generalized translation networks, zonal function networks, and ReLU$^r$ networks for \emph{non-integer $r$}. 
In particular, when applied to symmetric kernels, we do not require the kernels to be positive definite.
\item The space of target functions is not the native (variation) space, but may include ``rough'' functions.
\item Together with the approximation of functions by networks, we study the simultaneous approximation of their derivatives by the corresponding derivatives of the approximating networks. 
This sort of approximation is needed in many examples, optimal control in particular.
\item The method involved is a combination of both probabilistic and approximation theory techniques.
\end{itemize} 
Future directions of research in this direction would be, for example, obtain results where the unspecified constants involved are dependent only polynomially on the dimensions of the spaces involved, develop constructive tools for the networks which have the same expressive power, and extend the theory to approximation by general asymmetric dictionaries. 


\begin{thenomenclature} 
\nomgroup{A}
  \item [{$\BB(x,r)$}]\begingroup Closed ball of radius $r$, centered $x$\nomeqref {2.0}\nompageref{7}
  \item [{$\HH_j^q$}]\begingroup Space of all homogeneous, harmonic spherical polynomials, Example~\ref{uda:specialmanifold_sph}\nomeqref {2.8}\nompageref{8}
  \item [{$\lambda_k$, $\hat{\lambda}_\ell$}]\begingroup Sequence defined in Section~\ref{bhag:ddr}, typically eigenvalues of the Laplace-Beltrami operator\nomeqref {2.1}\nompageref{7}
  \item [{$\mathbf{G}_n, \mathbf{L}_n$}]\begingroup cf. \eqref{eq:lipcond}\nomeqref {3.4}\nompageref{15}
  \item [{$\mathcal{D}_G$}]\begingroup special operator associated with $G$, \eqref{eq:basicass}, \eqref{eq:gendfdef}\nomeqref {2.35}\nompageref{12}
  \item [{$\mathcal{E}_M(f)$}]\begingroup Degree of approximation of $f$ by $M$-term neural networks, \eqref{eq:neuraldegree}\nomeqref {1.3}\nompageref{4}
  \item [{$\mathcal{U}$}]\begingroup derivative-like operator, Section~\ref{bhag:smoothness}\nomeqref {2.31}\nompageref{11}
  \item [{$\mathfrak{c}$}]\begingroup counting measure on $\ZZ_+$\nomeqref {2.39}\nompageref{12}
  \item [{$\mu^*$}]\begingroup Distinguished probability measure, used with subscripts as needed\nomeqref {2.0}\nompageref{7}
  \item [{$\nu_G$}]\begingroup Measure associated with eignents, \eqref{eq:ass_measure}\nomeqref {2.38}\nompageref{12}
  \item [{$\omega_q$}]\begingroup volume of $\SS^q$\nomeqref {2.11}\nompageref{9}
  \item [{$\phi_k$, $\psi_k$}]\begingroup Orthonormal functions, typically eigen-funtions, cf. Section~\ref{bhag:ddr}\nomeqref {2.2}\nompageref{7}
  \item [{$\Phi_n$}]\begingroup diffusion polynomial kernels, \eqref{eq:kerndef}\nomeqref {2.14}\nompageref{9}
  \item [{$\pi^*$}]\begingroup coordinate map for upper hemisphere $\SS^q_+$, \eqref{eq:euclid_to_sphere}\nomeqref {4.1}\nompageref{17}
  \item [{$\Pi_n^q$}]\begingroup space of sperhical polynomials of degree $<n$\nomeqref {2.8}\nompageref{8}
  \item [{$\rho$}]\begingroup metric\nomeqref {2.0}\nompageref{7}
  \item [{$\sigma_n, \tau_j$}]\begingroup Reconstruction and analysis operators, \eqref{eq:opdef}, \eqref{eq:analopdef}\nomeqref {2.16}\nompageref{9}
  \item [{$\SS^q$}]\begingroup Unit sphere embedded in $\RR^{q+1}$, Example~\ref{uda:specialmanifold_sph}\nomeqref {2.8}\nompageref{8}
  \item [{$\TT^q$}]\begingroup torus of dimension $q$, Example~\ref{uda:specialmanifold_torus}\nomeqref {2.8}\nompageref{8}
  \item [{$\vert\nu\vert$}]\begingroup total variation measure for $\nu$\nomeqref {5.23}\nompageref{23}
  \item [{$\Xi$}]\begingroup Compact data space, Definition~\ref{def:ddrdef}\nomeqref {2.3}\nompageref{7}
  \item [{$\XX$}]\begingroup metric measure space\nomeqref {2.0}\nompageref{7}
  \item [{$\{Y_{\ell,k}\}_{k=1}^{d_j^q}$}]\begingroup orthonormal basis for $\HH_j^q$\nomeqref {2.8}\nompageref{8}
  \item [{$a^*$, $a_*$}]\begingroup maximum and minimum orders of derivative-like operators, \eqref{eq:minexponent}\nomeqref {3.2}\nompageref{15}
  \item [{$d_j^q$}]\begingroup dimension of $\HH_j^q$\nomeqref {2.8}\nompageref{8}
  \item [{$E_{p;n}$}]\begingroup Degree of approximation, \eqref{eq:degapproxdef}\nomeqref {2.12}\nompageref{9}
  \item [{$G(\ell,x,y)$}]\begingroup Asymmetric eignet kernel\nomeqref {2.35}\nompageref{12}
  \item [{$G_r$}]\begingroup Activation function for ReLU$^r$ networks \eqref{eq:reluact}\nomeqref {4.0}\nompageref{17}
  \item [{$H$}]\begingroup Band pass filter, Section~\ref{bhag:degapprox}\nomeqref {2.14}\nompageref{9}
  \item [{$p_j^{(q/2-1,q/2-1)}$}]\begingroup orthonormalized ultraspherical polynomials cf. \eqref{eq:ultraortho}\nomeqref {2.11}\nompageref{9}
  \item [{$S_\ell, S_n^*, \mathbf{S}_j$}]\begingroup Index sets, cf. \eqref{eq:indexset}\nomeqref {2.2}\nompageref{7}
  \item [{$T_n$}]\begingroup \eqref{eq:tndef}\nomeqref {3.2}\nompageref{15}
  \item [{$t_{k,\ell}$}]\begingroup Orthonormal functions on $SO(q+1)$, cf. \eqref{eq:gprepresent}\nomeqref {2.47}\nompageref{13}
  \item [{$V_\ell, \Pi_n, \mathbf{V}_j$, $\Pi_\infty$}]\begingroup Spaces of diffusion polynomials \eqref{eq:spacedef}\nomeqref {2.2}\nompageref{7}
  \item [{$W_{p;\gamma}(\XX)$}]\begingroup Sobolev approximation space \eqref{eq:targetsmooth}\nomeqref {2.21}\nompageref{10}
  \item [{$X^p$}]\begingroup $L^p$-closure of $\Pi_\infty$\nomeqref {2.13}\nompageref{9}

\end{thenomenclature}


\end{document}